%% file: main.tex
\documentclass[10pt]{article}

\usepackage{aithyra}

\input{math_commands.tex}
\definecolor{refpurple}{HTML}{6A3D9A}
\colorlet{aithyra-shade}{aithyra-teal!12}   
\definecolor{indgray}{HTML}{595959}   
\definecolor{gtpurple}{RGB}{75,45,140}  

\usepackage{wrapfig}
\usepackage{cancel}

\providecommand{\Nc}{\mathcal{N}}
\providecommand{\Id}{\mathrm{Id}}

\definecolor{gtpurple}{RGB}{75,45,140}  %
\newcommand{\gt}{\textcolor{gtpurple}{\texttt{\textbf{GT}}}\xspace}

\newcommand{\globalt}{%
  \textcolor{gtpurple}{\texttt{\textbf{Global\kern0.22em Transport}}}\xspace}

\definecolor{myred}{RGB}{215,48,39}
\definecolor{mygreen}{RGB}{26,152,80}
\newcommand{\cmark}{\textcolor{mygreen}{\ding{51}}}
\newcommand{\xmark}{\textcolor{myred}{\ding{55}}}
\newcommand{\halfmark}{\textcolor{gray}{\checkmark\kern-1.1ex\raisebox{.7ex}{\rotatebox[origin=c]{125}{--}}}}

\newcommand{\cut}[1]{}

\input{macros}

\title{Global Transport Couplings for Classifier-Free Guided Flows}
\date{}

\begin{document}

\cut{\noindent\includegraphics[height=0.4in]{figures/aithyra.png}}

\noindent{\color{aithyra-teal}\rule{\textwidth}{1pt}}
\vspace{0.05em}

\begin{center}
    {\sffamily\bfseries\huge Global Transport Couplings \\ for Classifier-Free Guided Flows}
\end{center}

\vspace{0.2em}
\noindent{\color{aithyra-teal}\rule{\textwidth}{1pt}}

\vspace{0.5em}

\begin{center}
    {\sffamily\bfseries
        Katarina Petrović\textsuperscript{1,2}\quad
        Zander W. Blasingame\textsuperscript{1}\quad
        Danyal Rehman\textsuperscript{1,3,4,5}\quad
        \.{I}smail \.{I}lkan Ceylan\textsuperscript{6,1,2}\quad
        Michael Bronstein\textsuperscript{1,2}\\
        Stephen Y. Zhang\textsuperscript{7\textdagger}\quad
        Lazar Atanackovic\textsuperscript{8,9,10\textdagger}\quad
        Alexander Tong\textsuperscript{1\textdagger}
    }\\[0.5em]
    {\sffamily
        \small
        \textsuperscript{1}AITHYRA \quad
        \textsuperscript{2}University of Oxford \quad
        \textsuperscript{3}Mila - Québec AI Institute \quad
        \textsuperscript{4}Université de Montréal
        \\
        \textsuperscript{5}Massachusetts Institute of Technology \quad
        \textsuperscript{6}Technische Universität Wien \quad
        \textsuperscript{7}Flatiron Institute \quad
        \textsuperscript{8}University of Alberta \\
        \textsuperscript{9}Alberta Machine Intelligence Institute \quad
        \textsuperscript{10}Canada CIFAR AI Chair

        {\textsuperscript{\textdagger}Equal advising}
    }
\end{center}

\vspace{0.5em}

\begin{tcolorbox}[
    enhanced,
    breakable,
    colback=aithyra-lightgrey!50,
    colframe=aithyra-teal,
    boxrule=1pt,
    arc=10pt,
    left=10pt, right=10pt, top=8pt, bottom=8pt,
]

{\centering\sffamily\bfseries\large\color{aithyra-teal} Abstract\par}

\vspace{0.6em}

{\sffamily

 Optimal-transport couplings have been shown to reduce training variance in unconditional flow models, but their role in conditional generation remains unclear. A natural approach constructs separate couplings for each condition, but this is impractical for large or continuous conditioning spaces found in modern image foundation models. We introduce \globalt (\gt), a global class-agnostic optimal-transport coupling, computed without class labels. \gt can associate different conditions with different regions of the source noise, and consequently worsens performance without guidance. However, when combined with classifier-free guidance (CFG), \gt consistently improves generation across domains, model scales, and sampling budgets. This reversal suggests that couplings for conditional flows should be evaluated both empirically and theoretically under the guided flow used at inference, rather than on unguided generation. We evaluate \gt over both discrete class and continuous text conditioned image generation across model scales, and investigate how coupling choice alters guided trajectories. These results identify coupling design in the guided flow setting as a simple training time axis to improve performance without modifying existing architectures, samplers, or guidance mechanisms. \looseness=-1
 
}

\vspace{0.8em}

{\small\sffamily
    \textbf{Correspondence:}
    \href{mailto:kpetrovic@aithyra.at}{\texttt{kpetrovic@aithyra.at}}\\
    \textbf{Date:} 6 October, 2026
}

\end{tcolorbox}

\vspace{-5pt}
\section{Introduction}
\vspace{-5pt}

\looseness=-1
Flow matching~\parencite{lipman2023flowmatchinggenerativemodeling, albergo2023building,peluchetti2023non} and diffusion models~\parencite{ho2020denoising, song2021scorebased} have emerged as powerful paradigms for generative modeling, achieving state-of-the-art generation results across a wide range of domains including images \parencite{Rombach_2022_CVPR, esser2024scaling, Ma2024SiTEF}, videos \parencite{ho2022video, blattmann2023stable, polyak2024movie}, and biological data \parencite{morehead2026flow, li2026flow}. In the simplest construction, source and data samples are paired independently. A widely adopted strategy for improving generation quality is by incorporating optimal transport (OT) \emph{couplings} between noise samples and data points~\parencite{pmlr-v202-pooladian23a, tong_conditional_2023}. OT pairs noise and data to minimize transport cost, reducing the training variance and enabling more efficient numerical integration and sampling efficiency of flow matching models ~\parencite{lu2026bidirectional,berthelot2026coupling,malnick2026optimaltransportflowmatching}.
 These benefits suggest that coupling design could also improve conditional generation. It is less obvious, however, what an appropriate coupling should be when every data sample is associated with a class, text prompt, or other condition.

The natural extension of this paradigm is to match source and data separately for each condition. Such class-conditional OT respects the common source distribution used to sample each class, but requires sufficiently large batches of examples sharing a condition~\parencite{chemseddine2025conditional, kerrigan2024dynamic, mousavi-hosseini2026flow,kong2026alignflow,cheng2025curse}. This becomes impractical when there are many classes and is infeasible when conditions, such as text-embeddings, are continuous. As a result, class-conditional OT has achieved limited adoption in the conditional generative setting.

There is, however, an additional confounder in the case of conditional generation. Conditional flow models are almost ubiquitously sampled with classifier-free guidance (CFG)~\parencite{ho2022classifier}, which combines conditional and unconditional velocities and is used to steer towards a condition and away from the unconditional density. Class-conditional OT has been shown to improve both the unguided conditional flow, and the guided flow~\parencite{cheng2025curse}. 
The natural extension of this paradigm to the \emph{conditional} generative setting is via \emph{class-conditional} OT, where one seeks an optimal coupling for each \emph{class}.
However, this procedure necessitates the computation of couplings for class-conditioned batches, which becomes impractical for large conditioning spaces frequently found in modern generative models, and infeasible in continuous conditioning spaces used in text-to-image models \parencite{esser2024scaling}. 
As a result, class-conditional OT has achieved limited adoption in the conditional generative setting, and yields modest gains for the added computational cost.


CFG trains a flow model to perform both conditional and unconditional generation, dropping the conditioning signal on some fraction of the data points; at inference time, guided samples are drawn by linearly composing the two velocity fields, modulated by a guidance scale.
As a result, CFG guided samples follow trajectories defined jointly by both conditional and unconditional velocity fields, rather than by the conditional model alone. Class conditional OT improves both the unconditional and conditional fields separately. A natural question then arises: \looseness=-1
\vspace{2pt}
\begin{mdframed}[style=MyFrame3]
    Can a coupling that considers both fields together produce a better \textit{guided} generator?
\end{mdframed}
\vspace{-3mm}

We answer in the affirmative with \globalt (\gt), which computes a single \textit{global} OT assignment across source and target samples without considering their conditions. \gt is straightforward to apply to both class labels and text conditioning, and requires no change to the model, sampler, or CFG rule. However, \gt comes with a tradeoff: Although the coupling preserves the overall source marginal, the conditional source distribution do not need to be consistent with the unconditional source distribution. Following from this mismatch, \gt performs worse in unguided conditional sampling in our experiments. Under CFG, however, the ordering reverses with \gt improving guided generation across domains, model scales, and sampling budgets.

To investigate the cause of this reversal, we study the coupling choice from the perspective of the \textit{guided} flow. We find \gt reduces the \textit{prediction gap}~\parencite{wang2025towards,cai2026improving}, the difference between the conditional and unconditional flows. This reduces a pattern we term the \textit{yo-yo effect}, where independent flows exhibit pronounced contraction, re-expansion, and overshooting along guided trajectories with increasing guidance strength. Our contributions are:




\begin{itemize}[leftmargin=1.2em, itemsep=2pt, topsep=2pt, parsep=0pt, partopsep=0pt]
    \item We propose \gt, a practical, low-overhead coupling strategy for conditional generation that requires no change to architectures, samplers, or guidance mechanisms.
    \item We demonstrate that unguided conditional performance does not always correlate with guided performance, and analyze how coupling choice affects guided trajectories and the prediction gap.
    \item We show that \gt consistently yields competitive generation performance on class-conditioned image generation (e.g.\ FID 1.91 on ImageNet-256), text-conditioned image generation, and single-cell gene expression data, across model scales and sampling budgets, and outperforms both independent and class-conditional couplings. \looseness=-1
\end{itemize}

\section{Background and preliminaries}
\xhdr{Flow Matching}
Flow matching (FM)~\parencite{lipman2023flowmatchinggenerativemodeling, albergo_stochastic_2023,liu2023flow,peluchetti2023non} allows continuous time transport between distributions: the algorithm trains a velocity field that evolves samples from a source distribution $p_0$ to a target data distribution $p_1 = p_{\rm data}$, where $p_0$ is typically chosen to be an easy-to-sample distribution such as $\mathcal{N}(0,I_d)$. \looseness=-1

Posit a continuous-time transport, i.e. a \emph{probability path}, $(p_t)_{t\in[0,1]}$ with the prescribed boundary conditions $(p_0, p_1)$. Any such $p_t$ is generated by a marginal velocity field $u_t$ such that $\partial_t p_t + \nabla \cdot (p_t u_t) = 0$. 
FM seeks to approximate $u_t$ with a model $u_t^\theta$ by least-squares regression, minimizing \(\mathcal{L}_{\mathrm{FM}}(\theta)
    =
    \mathbb{E}_{t\sim U[0, 1],\, x_t\sim p_t}
    \left[
        \bigl\|u_t^\theta(x_t)-u_t(x_t)\bigr\|_2^2
    \right]\).
This objective is intractable, as in practice $u_t$ is unavailable in closed form. Conditional FM \parencite{tong_conditional_2023, albergo_stochastic_2023} instead introduces a coupling \(\pi\in\Pi(p_0,p_1)\) and a conditional probability path \(p_t(\cdot\mid x_0,x_1)\) generated by tractable conditional velocity field \(u_t(\cdot\mid x_0,x_1)\).
Taking expectations yields the CFM objective, which trains a neural velocity field $u^\theta$ by regression onto the closed-form conditional velocity $u_t(x_t \mid x_0, x_1)$: \looseness=-1
\begin{equation}
\label{eq:L_cfm}
    \mathcal{L}_{\mathrm{CFM}}(\theta)
    =
    \mathbb{E}_{t,\,(x_0,x_1)\sim\pi,\, x_t\sim p_t(\cdot\mid x_0,x_1)}
    \left[
        \bigl\|
        u_t^\theta(x_t)-u_t(x_t\mid x_0,x_1)
        \bigr\|_2^2
    \right].
\end{equation}
\cut{Crucially, the two objectives differ only by a constant independent of
\(\theta\)~\parencite{lipman2023flowmatchinggenerativemodeling}, so the gradients are equivalent.
Performing inference requires integrating neural velocity field over time from start time $s \in [0, 1]$ to target time $t \in [0,1]$ i.e.
\begin{equation}
    x_{t|s} = x_s + \int_s^t u^\theta(x_\tau) \;\mathrm d \tau; \quad x_\tau \sim p_\tau
\end{equation}
which can be discretized and integrated numerically which can be quite computationally expensive.} 

Once trained, samples from $p_1$ can be generated by first drawing from the source $x_0 \sim p_0$, then numerically integrating the ODE $\dot{x}_t = u_t(x_t)$. 
\cut{
requires integrating neural velocity field over time from start time $s \in [0, 1]$ to target time $t \in [0,1]$ i.e. $x_{t|s} = x_s + \int_s^t u^\theta(x_\tau) \;\mathrm d \tau$, methods such as consistency models~\parencite{song2023consistencymodels}, meanflows~\parencite{geng2025meanflowsonestepgenerative}, and flow maps~\parencite{boffi2025buildconsistencymodellearning}, frame the problem as learning to predict the integration directly that is the parameterize
\begin{equation}
    \Phi_{s,t}^\theta(x_s) \approx x_s + \int_s^t u^\theta(x_\tau)\; \mathrm d\tau
\end{equation}
which can either be trained together with the flow~\parencite{geng2025meanflowsonestepgenerative, boffi2025buildconsistencymodellearning} or distilled from it~\parencite{sabour2026align, lee2026decoupled}.}

\xhdr{Classifier-free Guidance}
Conditional generation, present in many applications such as text-to-image or text-to-video synthesis,
requires sampling given a prompt \(c\in\mathcal{C}\), achieved by
learning a prompt-conditioned velocity field \(u_t(x_t\mid c)\) via the
analogue of~\eqref{eq:L_cfm},
\begin{equation}
    \label{eq:L_cfm_cond}
    \mathcal{L}^{c}_{\mathrm{CFM}}(\theta)
    =
    \mathbb{E}_{t,\,(x_0,x_1,c)\sim\pi,\, x_t\sim p_t(\cdot\mid x_0,x_1)}
    \left[
        \bigl\|
        u_t^\theta(x_t, c)-u_t(x_t\mid x_0,x_1, c)
        \bigr\|_2^2
    \right].
\end{equation}
Classifier-free guidance (CFG)~\parencite{ho2022classifier} improves sample
quality and prompt alignment by sampling from a linear combination of the
marginal and conditional fields,
\begin{equation}\label{eq:cfg_main}
    \widetilde{u}_t^{w}(x_t \mid c)
    =
    u_t(x_t) + w\, g(x_t, t, c), \quad
    g(x_t,t,c) = u_t(x_t\mid c)-u_t(x_t)
\end{equation}
where \(w\ge 1\) and \(g\) are the guidance weight and vector respectively. In practice, a single model is trained across all conditions, with \(c\) replaced by a null token \(\varnothing\) with some probability, so that
\(u_t^\theta(x_t,\varnothing)\) estimates the marginal field. Appendix~\ref{appendix: cfg_theory} reviews the theory underlying CFG.

\xhdr{Optimal Transport} Optimal transport conditional flow matching (OT-CFM)~\parencite{tong_conditional_2023,pmlr-v202-pooladian23a,mousavi-hosseini2026flow,kong2026alignflow} replaces the independent coupling $\pi = p_0 \otimes p_1$ with approximations of the Euclidean OT coupling $\pi_{\rm OT}$, defined as the solution to \looseness=-1
\begin{equation}
\label{eq: OT}
    \pi_{\rm OT} = \text{argmin}_{\pi \in \Pi(p_0, p_1)} \mathbb{E}_{(x_0, x_1) \sim \pi} \| x_0 - x_1 \|^2_.
\end{equation}
The resulting trajectories are straighter in practice and therefore easier to integrate in the few-step sampling regime, shown to achieve modest but consistent gains in the unconditional setting.

\xhdr{Class-conditional Optimal Transport}
Class-conditional generation involves sampling from data distribution $p_1(x\mid c)$ conditioned on selected class $c\in \mathcal{C}$.
A na\"ive adaptation of minibatch OT-CFM to the class-conditional setting involves solving $|\mathcal{C}|$ independent OT problems, one for each class $c \in \mathcal{C}$. 
However, this quickly becomes infeasible when the number of classes $|\mathcal{C}|$ becomes even moderately large (for instance, ImageNet contains 1000 classes), and is not possible in settings where the condition itself is a continuous latent state, as is the case for text-to-image generation for example. \parencite{kerrigan2024dynamic, chemseddine2025conditional, cheng2025curse} proposed to approximate this by minimizing 
\begin{equation}
    \pi_{\rm C^2OT} = \text{argmin}_{\pi \in \Pi(p_0, p_1)} \mathbb{E}_{((x_0, c_0), (x_1, c_1)) \sim \pi} \left[ \| x_0 - x_1 \|^2 + \beta \|c_0 - c_1 \|^2 \right]
\end{equation}
where $\beta$ is adjusted to fit the scale and $c_{0, 1}$ denotes the class associated with $x_{0, 1}$. This allows the use of much smaller batches in practice, but with unknown degradation, and thus far has not been widely adopted in conditional generation settings, where independent couplings are almost universal.\looseness=-1 

\section{Global Transport}
\begin{wrapfigure}{r}{0.5\textwidth}
  \vspace{-2.1\intextsep}
  \begin{minipage}{0.5\textwidth}
  \small
  \setlength{\abovecaptionskip}{2pt}
  \setlength{\belowcaptionskip}{0pt}
  \begin{algorithm}[H]
    \caption{\gt training step}
    \label{alg:gtfm}
    \begin{algorithmic}
    \State $x_0 \sim p_0,\; (x_1, c) \sim p_1,\; t \sim \mathcal{U}(0, 1)$
    \State $\tilde{c} \gets \varnothing$ w.p.\ $p_{\text{uncond}}$, else $\tilde{c} \gets c$
    \State $\hat{x}_0 \gets \pi_{\gt}(x_0);\quad x_t \gets t x_1 + (1 - t)\, \hat{x}_0$
    \State $\theta \gets \mathrm{Update}\big(\theta, \nabla_\theta \| u_t^\theta(x_t, \tilde{c}) - (x_1 - \hat{x}_0)\|^2\big)$
    \end{algorithmic}
  \end{algorithm}
  \end{minipage}
  \vspace{-1\intextsep}
\end{wrapfigure}

We next introduce \globalt (\gt) a practical and low overhead application of optimal transport which improves performance in the class-conditioned setting. \gt computes an approximate optimal transport coupling without encouraging or constraining assignments to preserve class labels. The method does not change the architecture nor inference algorithm and adds only a training-time assignment step. We first describe the overall objective, then two approximations either using (1) a mini-batch approximation~\parencite{fatras2021minibatch,tong_conditional_2023,pmlr-v202-pooladian23a} or (2) a semi-discrete approximation~\parencite{mousavi-hosseini2026flow,kong2026alignflow}.

\xhdr{General objective} Given noise samples $x_0 \sim p_0$ and class conditional data samples $(x_1, c) \sim p_1$, for the optimal transport coupling $\pi_{OT}$ as defined in \eqref{eq: OT}, and a standard linear flow matching path $x_t = t x_1 + (1-t) x_0$ where the noise is rearranged to match the closest data points, ignoring the labels $c$. This results in the conditional flow matching objective:
\begin{equation}
    \gL_{\rm GT}(\theta) = \mathbb{E}_{t\sim U[0,1], (x_0, x_1) \sim \pi_{\rm OT}} \| u^\theta_t(x_t, \tilde{c}) - (x_1 - x_0) \|^2
\end{equation}
where $\tilde{c}$ is set to $\varnothing$ with probability $p_{\text{uncond}}$ and $c$ otherwise. The full algorithm appears in Algorithm~\ref{alg:gtfm}.

\xhdr{Minibatch-OT Implementation}
We compute the assignment using the Hungarian algorithm with squared Euclidean cost in the model input space. For a batch of size \(B\), the cost matrix requires \(O(B^2)\) memory and the exact assignment \(O(B^3)\) time. This is generally dwarfed by network evaluation time on large systems and is negligible (usually $<2\%$ overhead) for modern workflows. However, it can be expensive for large batch sizes, if this setting is desired a regularized~\parencite{tong_simulation-free_2023,cuturi2013sinkhorn,zhang2026on}, or semi-discrete approach may be more appropriate.

\xhdr{Semi-discrete OT Implementation} Building off of recent work on Flow matching with semi-discrete optimal transport couplings, we can also use an approximate semi-discrete OT implementation. This approach first performs an expensive preprocessing step to optimize a semi-dual potential function over all discrete datapoints, which can then be used to calculate couplings with any member of the continuous noise measure. 

\xhdr{The conditional source distribution mismatch} Ignoring conditions while constructing the coupling changes the distribution of $(x_0, x_1)$ pairs the conditional model learns from. To illustrate this, we consider the full batch, idealized setting using the true OT coupling $\pi_{OT}$, with associated transport map $T: p_0 \to p_1$. Draw $x_0 \sim p_0$ and set $x_1 = T(x_0)$ with condition $C$. Then the source points that are paired with points with condition $c$ do not need to be distributed like $p_0$ as different conditions may be paired with source points from different regions of the prior. In other words, although $x_0 \sim p_0$, the conditional distribution $x_0 \mid C = c$ need not equal $p_0$. This tells us that under the unguided setting ($w=1$), the standard class conditioned inference procedure starting with samples from $p_0$ will not necessarily land at $p_1 | c$, and is only guaranteed to if the initial samples are drawn from the intractable distribution $x_0 \mid C = c$.


\cut{\begin{algorithm}[t]
  \caption{Global Transport Training}
  \label{alg:gtfm}
\begin{algorithmic}
\State {\bfseries Input:} Samplable distributions $p_0,p_1$, network $u^{\theta}$.
\While{Training}
\State $x_0 \sim p_0(x_0); \quad (x_1, c) \sim p_1(x_1, c); \quad t \sim \mathcal{U}(0, 1)$
\State $\tilde{c} \gets \varnothing$ with probability $p_{\text{uncond}}$
\State $\pi \gets \text{OT}(x_0, x_1)$
\State $x_t \gets t x_1 + (1 - t) \pi(x_0)$
\State $\theta \gets \mathrm{Update}(\theta, \nabla_\theta \| u_t^\theta(x_t, \tilde{c}) - (x_1 - \pi(x_0))\|^2$
\EndWhile
\State \Return $u^\theta$
\end{algorithmic}
\end{algorithm}}

\begin{figure}[h]
    \centering
    \includegraphics[width=\linewidth]{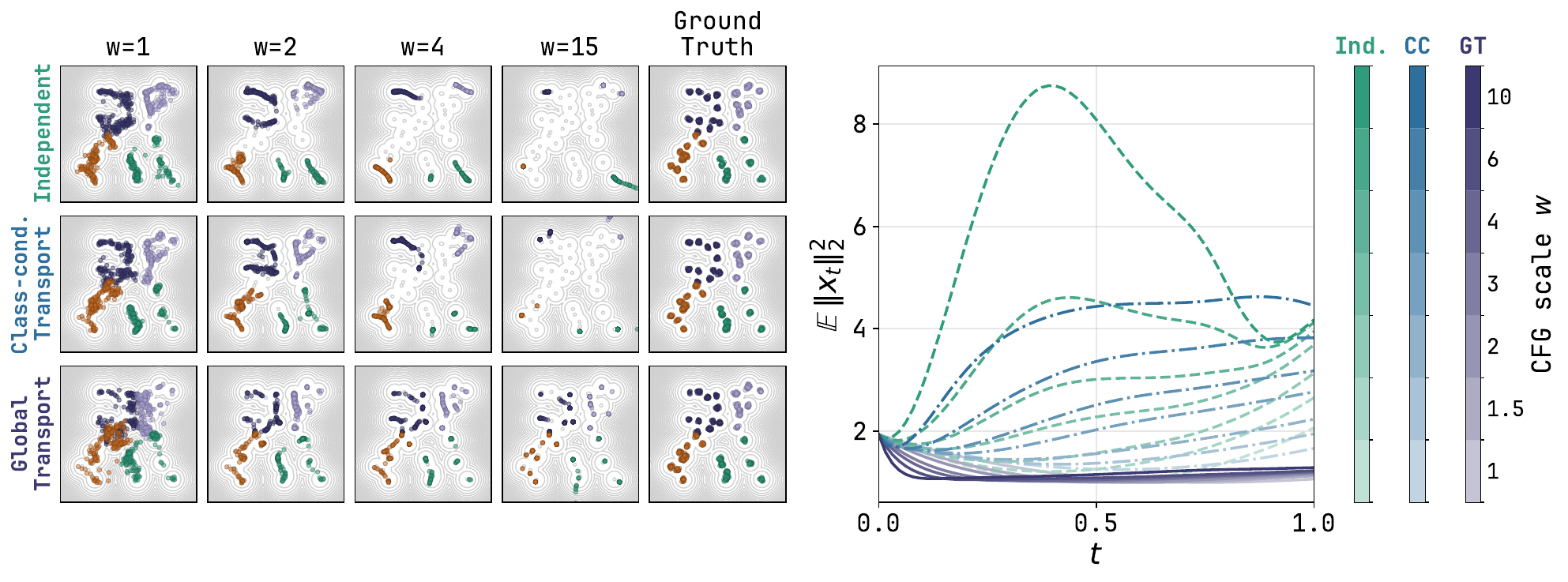}
    \vspace{-20pt}
    \caption{\textbf{2-dimensional depiction of the ``yo-yo'' effect.} This effect pushes samples away from the center at high guidance scales for the independent and class-conditional OT couplings, degrading generative performance. Global Transport fixes this and is stable under increasing guidance $w$.}
    \label{fig:toy-guidance}
\end{figure}

\cut{
This effect leads to samples being pushed away from the center at high guidance scales for the independent and class-conditional OT couplings, leading to poor generative performance. This is fixed in Global couplings which is stable under increasing guidance $w$.
}
\section{Coupling Choice With Classifier-Free Guidance}
\label{sec: gt}
The conditional source distribution mismatch described above suggests that \gt may be a poor choice for unguided conditional generation $(w=1)$, which we observe in a 2-dimensional 40-Gaussian mixture model example with four classes (Figures~\ref{fig:toy-guidance} and \ref{fig:ground_truth_toy} with details in Appendix~\ref{appendix: case_study_toy}). However, the result changes under guidance $(w > 1)$. This guidance dependent effect motivates studying the unconditional and conditional fields together as they are used during inference, rather than judging a coupling by its unguided conditional generation performance.

\begin{figure}[t]
    \centering
    \includegraphics[width=\linewidth]{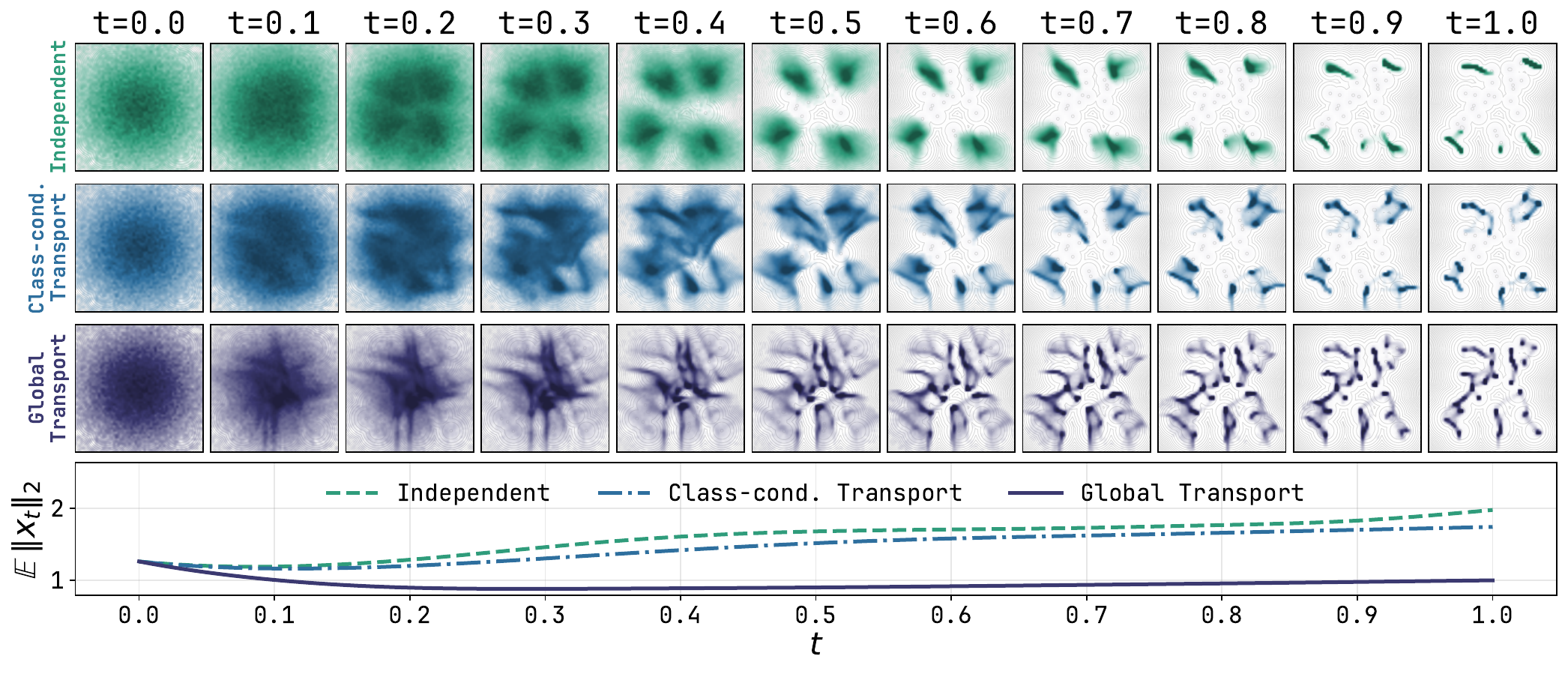}
    \caption{Density over time for Independent, Class conditional, and \gt couplings. Independent ``yoyo''s, contracting towards the class origin then expanding outwards, whereas \gt is stable.\looseness=-1}
    \label{fig:toy_yoyo}
\end{figure}
\subsection{Geometry of guided trajectories}

Let $X_t^{(w)}$ denote a sample generated by integrating
$\dot X_t^{(w)}
=\widetilde{u}_t^{\,w}(X_t^{(w)}\mid c)$
from $X_0^{(w)}\sim p_0$. To track its radial position over time, we measure
\begin{equation}
    Y_w(t)=\mathbb{E} \|X_t^{(w)}\|_2.
    \label{eq:guided-state-norm}
\end{equation}
Unlike the distribution of linear interpolants used in training,
$Y_w(t)$ describes the radial position of trajectories produced by the guided sampler.

In our GMM example, trajectories trained with independent
pairing contract before expanding toward their class modes (Figure~\ref{fig:toy_yoyo}).
As $w$ increases, the re-expansion becomes more pronounced and can
overshoot the target regions (Figure~\ref{fig:toy-guidance}). We refer to this contraction followed by re-expansion as the \emph{yo-yo effect}. Class-conditional OT reduces some aspects of this behavior, while \gt exhibits less
pronounced contraction and overshoot over the guidance scales we
test (Figures~\ref{fig:toy-guidance}
and~\ref{fig:toy_analysis}).

\begin{figure}[h]
    \centering
    \includegraphics[width=\linewidth]{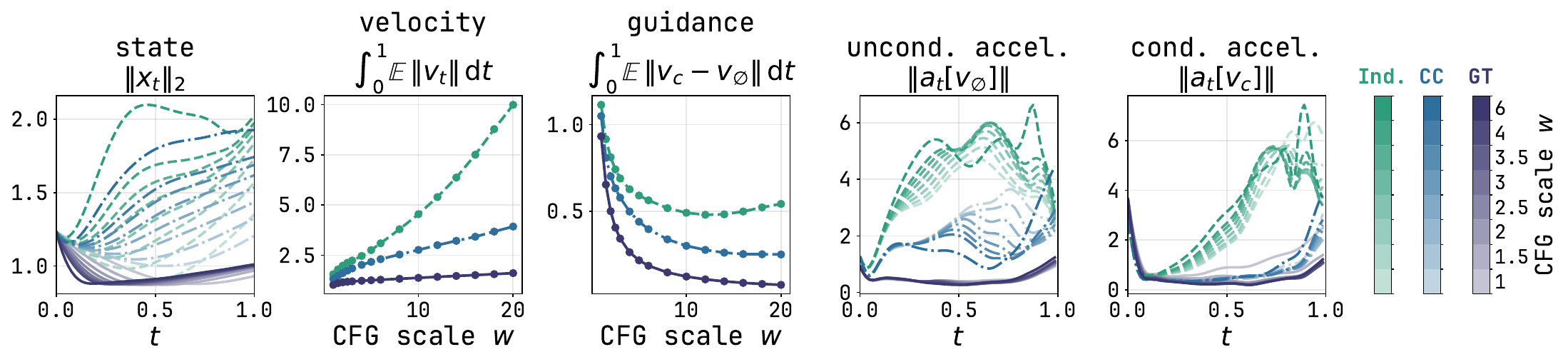}
    \caption{\textbf{Guided dynamics on the 40-GMM} across guidance scales $w$. From left: state norm along guided trajectories; integrated guided velocity norm; integrated prediction gap $\|v_c - v_\varnothing\|$; unconditional and conditional material acceleration $\|a_t[v]\|$.}
    \label{fig:toy_analysis}
\end{figure}

\subsection{Coupling controls the prediction gap}
\label{sec:prediction-gap}

Recall from \eqref{eq:cfg_main} the guidance vector $g(x_t, t, c) = u_t(x_t | c) - u_t(x_t)$ measures the gap between unconditional and conditional velocity fields. Following \textcite{wang2025towards} we call $\|g(x,t,c)\|^2$ the \emph{prediction gap} and ask how the training coupling constrains its size.

Let $\pi$ be the joint distribution of matched source points $x_0$,
data points $x_1$, and conditions $C$. Its $(x_0,x_1)$ marginal
couples $p_0$ and $p_1$. Define the training interpolant
$x_t=(1-t)x_0+tx_1$ and its coupling cost
$\operatorname{Cost}(\pi)=\mathbb{E}_{\pi}\|x_1-x_0\|^2$. 

\begin{proposition}[Transport cost bounds the on-path prediction gap]
\label{prop:gap-transport}
Suppose $u_t(x)$ and $u_t(x\mid c)$ are the population
squared-error minimizers for the common flow-matching target
$x_1-x_0$ under $\pi$. Then
\begin{equation}
\label{eq:gap-transport}
\mathbb{E}_{t, \pi}
\left[\|g(X_t,t,C)\|^2\right] 
\leq 
\operatorname{Cost}(\pi) - \mathcal{W}_2^2(p_0,p_1).
\end{equation}
\end{proposition}

The proof in Appendix~\ref{appendix: g_proof} first expresses the expected prediction gap as the difference between the optimal unconditional and conditional flow-matching regression errors. It then bounds the integrated unconditional error by the excess quadratic transport cost. A lower-cost coupling can therefore give a tighter upper bound on the average prediction gap along its training interpolants. We note that this bound holds strictly along the linear training interpolants and does not govern the distribution of states visited during guided inference $(w>1)$, which we verify empirically. We also note that at the optimal transport limit, the bound contracts to zero because deterministic transport renders class conditioning redundant on the training support. In practice, finite mini-batch OT retains a non-zero prediction gap while regularizing the velocity fields. Further, we can show a hierarchy of coupling costs between different methods compared in this paper, which we do in the following proposition.

\begin{proposition}
\label{cor:coupling-order}
For all $\beta$ the coupling costs satisfy
\begin{equation}
    \operatorname{Cost}(\pi_{\gt}^\star)
    \leq
    \operatorname{Cost}(\pi_{C^2OT(\beta)}^\star)
    \leq
    \operatorname{Cost}(\pi_{\mathrm{ind}}).
\end{equation}
with the lower equality achieved at $\beta=0$.
\end{proposition}
which establishes that the \gt algorithm has a tighter upperbound on the prediction gap than C$^2$OT (for $\beta > 0$) which has a tighter upper bound on the prediction gap than independent couplings. 

While these results establish tighter upperbounds on the prediction gap, this does not establish causality between lower prediction gaps and improved performance. We use the bound to motivate the
prediction-gap measurements in Figure~\ref{fig:toy_analysis},
not as an explanation of generation performance on its own. \looseness=-1

\section{Experiments}
We now demonstrate the effectiveness of \gt on conditional generation in several different settings including image generation and single-cell experiments.
We first show that \gt improves the performance of classifier-free guided class-conditional generation on ImageNet~\parencite{deng2009imagenet} ($256\times256$) when training flow matching models and distilling them into flow maps. Then we move to the case of continuous conditioned models in the text-to-image setting, before finally demonstrating the impact on single-cell data generation guided by cell type labels across multiple single-cell datasets. We denote the minibatch version of \gt-MB as \gt, and the semi-discrete OT version of \gt as \gt + SD-OT.

\newsavebox\benchtable

\begin{table*}[h]
\centering
\caption{\textbf{Left:} FID$\downarrow$ comparison against baselines. \textsuperscript{\textdagger}denotes results obtained with guidance interval~\parencite{kynkaanniemi2024applying}. \textbf{Right:} Curated class-conditional samples from SiT-XL/2 + \gt on ImageNet-256.}
\label{tab:imagenet256_benchmark}
\vspace{-0.4em}
\sbox\benchtable{%
\scriptsize
\setlength{\tabcolsep}{3.2pt}%
\begin{tabular}{@{}llclr@{}}
\toprule
Method & NFE & CFG & Params & FID $\downarrow$ \\
\midrule
\multicolumn{5}{@{}l}{\textcolor{gray}{\textbf{GANs / Normalizing Flows / Autoregressive models}}} \\
\textcolor{gray}{StyleGAN-XL~\parencite{sauer2022styleganxlscalingstyleganlarge}}
& \textcolor{gray}{1} & \textcolor{gray}{\xmark} & \textcolor{gray}{166M} & \textcolor{gray}{2.30} \\
\textcolor{gray}{STARFlow~\parencite{gu2025starflow}}
& \textcolor{gray}{1} & \textcolor{gray}{\xmark} & \textcolor{gray}{1.4B} & \textcolor{gray}{2.40} \\
\textcolor{gray}{VAR-$d$30~\parencite{tian2024visual}}
& \textcolor{gray}{$20$} & \textcolor{gray}{\cmark} & \textcolor{gray}{2B} & \textcolor{gray}{1.92} \\
\textcolor{gray}{MAR-H/2~\parencite{li2024autoregressive}}
& \textcolor{gray}{$64$} & \textcolor{gray}{\cmark} & \textcolor{gray}{943M} & \textcolor{gray}{1.55} \\
\midrule
\multicolumn{5}{@{}l}{\textbf{Diffusion / Flow models}} \\
ADM~\parencite{dhariwal2021diffusion}
& $500$ & \cmark & 554M & 10.94 \\
LDM~\parencite{Rombach_2022_CVPR}
& $500$ & \cmark & 400M & 3.60 \\
RIN~\parencite{jabri2022scalable}
& 1000 & \xmark & 410M & 3.42 \\
SimDiff~\parencite{hoogeboom2023simple}
& $1024$ & \cmark & 2B & 2.77 \\
U-ViT-H/2~\parencite{bao2023all}
& $100$ & \cmark & 501M & 2.29 \\
DiT-XL/2~\parencite{Peebles_2023_ICCV}
& $500$ & \cmark & 675M & 2.27 \\
SiT-XL/2\textsuperscript{\textdagger}~\parencite{Ma2024SiTEF}
& $500$ & \cmark & 675M & 2.06 \\
\rowcolor{aithyra-shade} SiT-XL/2 + \gt \textsuperscript{\textdagger} \textbf{(ours)}
& $500$ & \cmark & 675M & \textbf{1.91} \\
\bottomrule
\end{tabular}}%
\begin{minipage}[t]{0.55\textwidth}
\vspace{0pt}
\centering
\usebox\benchtable
\end{minipage}%
\hfill%
\begin{minipage}[t]{0.42\textwidth}
\vspace{3pt}
\centering
\includegraphics[height=\dimexpr\ht\benchtable+\dp\benchtable\relax,width=\linewidth,keepaspectratio]{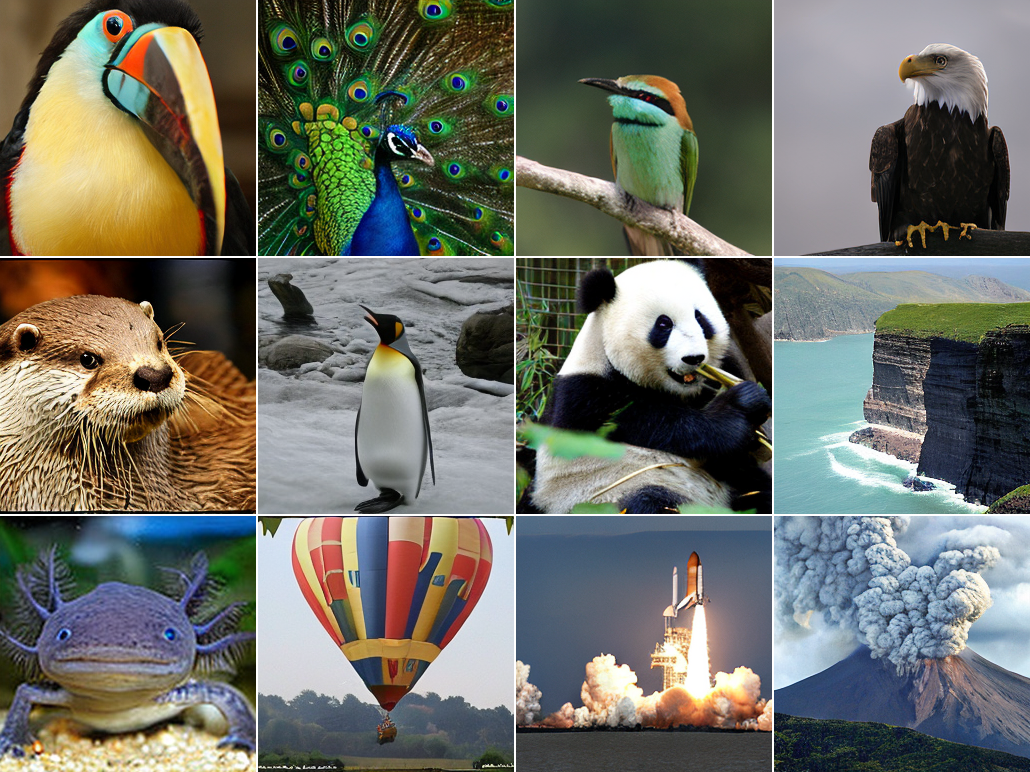}
\end{minipage}
\end{table*}

\cut{
\begin{figure}[h]
    \centering
    \includegraphics[width=0.7\linewidth]{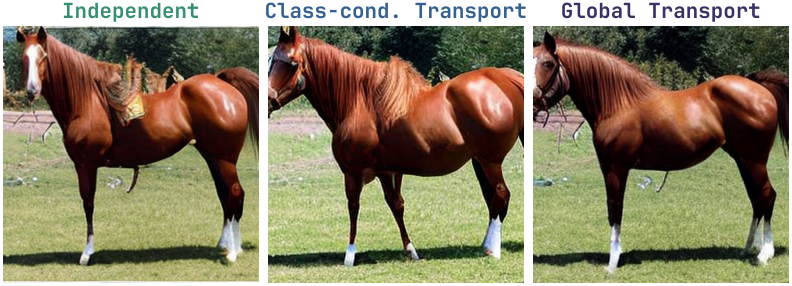}
    \caption{Class-conditional generation for indpendent, class-conditional and \gt transport for class 339 (\texttt{sorrel}) generated with SiT-L/2 (Euler, 64 steps) at guidance scale $w=2.5$.}
    \label{fig:generation_couplings}
\end{figure}
}

\begin{figure}[b]
    \centering
    \resizebox{\linewidth}{!}{%
        \raisebox{-0.5\height}{\includegraphics[height=4.6cm]{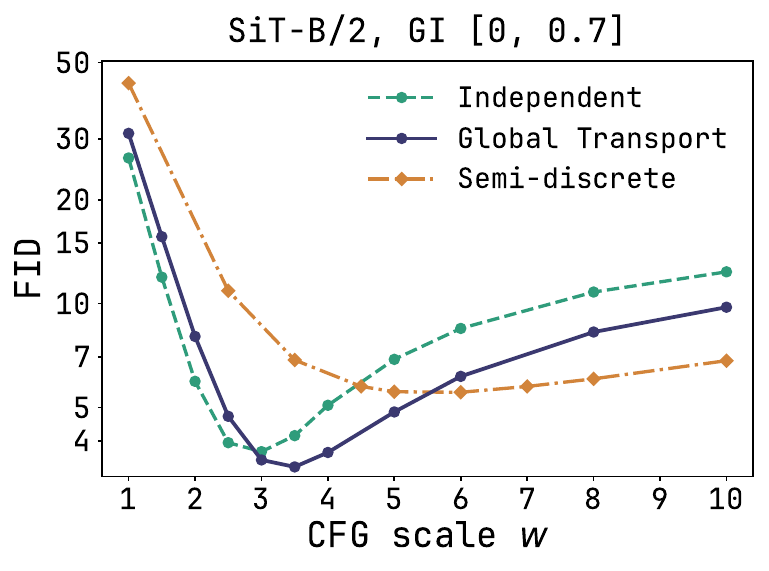}}%
        \hspace{0.4cm}%
        \raisebox{-0.5\height}{\includegraphics[height=4cm]{sit_l_2_cfg2.5_c339_i2_ind_ccot_ot.pdf}}%
    }
    \caption{\textbf{Left:} Guidance scale $w$ sweep for SiT-B/2 (Euler-64 steps) FID\,$\downarrow$ under a tuned guidance interval $[0, 0.7]$. \textbf{Right:} Class-conditional samples for class 339 (\texttt{sorrel}) from SiT-L/2 (Euler-64 steps, $w=2.5$) under independent, class-conditional, and \gt transport.}
    \label{fig:fid_cfg}
\end{figure}

\subsection{Class- and text- conditioned image generation on ImageNet-256} 
We follow the hyperparameter set-up of \textcite{Ma2024SiTEF} to train a SiT model on conditional generation across B/2, L/2 and XL/2 model scales using independent, class-conditional and \gt transport.
We report the Fr\'echet Inception Distance (FID) \parencite{heusel2017gans} and FD\textsubscript{DINOv2} \parencite{stein2023fddino} metrics to measure the distributional distance between the generated and real distributions. 
We highlight our best configuration in \cref{tab:imagenet256_benchmark} showing an improvement from the standard SiT-XL/2 trained with independent couplings to ours trained with the \gt strategy, improving from 2.06 to 1.91 in FID.
This one result, however, lays in front of a more interesting story.
\begin{wraptable}{r}{0.66\textwidth}
\caption{\small Comparison of different couplings across various NFEs and model sizes reported in FID and FD\textsubscript{DINOv2}. We swept for the optimal guidance strength $w^*$ for each setting and report the best result.}
\label{tab:sit_best_w}
\centering
\scriptsize
\begin{tabular}{ll ccc ccc}
\toprule
\multirow{2}{*}{Model} & \multirow{2}{*}{Coupling}
& \multicolumn{3}{c}{FID\,$\downarrow$}
& \multicolumn{3}{c}{FD\textsubscript{DINOv2}\,$\downarrow$} \\
\cmidrule(lr){3-5}\cmidrule(lr){6-8}
& & 16 & 32 & 64 & 16 & 32 & 64 \\
\midrule
& Independent     & 5.88 & 5.33 & 5.16 & 137.3 & 132.9 & 132.3 \\
& Class-cond.\ OT & 5.79 & 5.20 & 5.08 & 136.4 & 132.2 & 131.4 \\
\rowcolor{aithyra-shade}\cellcolor{white}
& \gt + SD-OT     & \textbf{5.11} & 4.90 & 4.85 & 199.6 & 194.6 & 184.7 \\
\rowcolor{aithyra-shade}\cellcolor{white}\multirow{-4}{*}{\shortstack[l]{SiT-B/2\\(130M)}}
& \gt             & 5.15 & \textbf{4.35} & \textbf{4.15}
                  & \textbf{123.6} & \textbf{118.0} & \textbf{116.4} \\
\midrule
& Independent     & 5.50 & 5.74 & 2.95 & 91.1 & 84.6 & 83.3 \\
& Class-cond.\ OT & 5.47 & 5.68 & 2.98 & 102.0 & 99.2 & 95.0 \\
\rowcolor{aithyra-shade}\cellcolor{white}\multirow{-3}{*}{\shortstack[l]{SiT-L/2\\(459M)}}
& \gt             & \textbf{3.94} & \textbf{3.84} & \textbf{2.69}
                  & \textbf{83.5} & \textbf{78.7} & \textbf{73.5} \\
\midrule
& Independent     & 3.95 & 3.00 & 2.78 & 84.2 & 78.7 & 77.8 \\
\rowcolor{aithyra-shade}\cellcolor{white}\multirow{-2}{*}{\shortstack[l]{SiT-XL/2\\(675M)}}
& \gt             & \textbf{3.22} & \textbf{2.69} & \textbf{2.59}
                  & \textbf{72.7} & \textbf{69.1} & \textbf{68.4} \\
\bottomrule
\end{tabular}
\end{wraptable}
Namely, with \gt the guidance scale $w$ can be pushed to much larger values before performance begins to degrade, enabling us to push far more aggressively into high guidance strength regimes.
To illustrate, consider \cref{fig:fid_cfg} where we compare the performance at various guidance scales $w$ along with applying the guidance interval technique \parencite{kynkaanniemi2024applying}.
We notice that for the model trained with \gt couplings we can use markedly higher guidance scales than the model trained with independent couplings; along with having an overall better minimum w.r.t.\ FID.
Observe that interval tuning improves all the couplings and shifts their optima to larger $w$, whilst \gt obtains the best global optima among the coupling strategies. 
Notably, the gap between \gt and independent couplings grows as $w$ increases. Further, observe that semi-discrete \gt which is closer to the exact OT map is the most robust at higher guidance scales, however does not reach the overall global minima FID.
To complement \cref{fig:fid_cfg} we report the best results across the different coupling strategies under several model sizes, swept over different guidance scales (per metric) in \cref{tab:sit_best_w}.
Observe that across all NFE and model sizes the \gt couplings obtain the strongest performance yielding a noticeable improvement. For model and training configurations please refer to appendix~\ref{appendix: imagenet}.\looseness=-1

\xhdr{Distillation into a flow map} We next assess whether \gt couplings can serve as a distillation strategy for few-step generators. Following~\textcite{lee2026decoupled}, we distill a SiT-B/2 and SiT-L/2 flow matching teacher into a flow map with using the meanflow distillation objective~\parencite{geng2026mean}.
The teacher is pretrained with either independent or \gt couplings and the student is distilled with \gt in both cases. On SiT-B/2, observe that applying \gt only at distillation (i.e., from an independently pretrained teacher) matches pretraining with \gt throughout, indicating that \gt is effective purely as a distillation strategy for stronger few-step generators outperforming other coupling choices. \looseness=-1
\begin{table}[h]
\centering
\small
\caption{\textbf{Left}: FID\,$\downarrow$ on ImageNet-256 for DMF distilled with independent, class-conditional and \gt couplings, at B/2 and L/2 scale. \textbf{Right}: Curated samples from \gt DMF-B/2 at 4-NFE.\looseness=-1}
\label{tab:b2_case_study}
\label{tab:dmf_l2}
\vspace{0.6em}
\footnotesize
\setlength{\tabcolsep}{4pt}
\begin{tabular}{@{}llccc@{\hspace{1.2em}}ccc@{}}
  \toprule
  & & \multicolumn{3}{c}{Teacher (NFE)} & \multicolumn{3}{c}{Student (NFE)} \\
  \cmidrule(lr){3-5} \cmidrule(l){6-8}
  Model & Coupling & 16 & 32 & 64 & 1 & 2 & 4 \\
  \midrule
  & Independent            & 5.88 & 5.33 & 5.16 & 5.82 & 5.19 & 5.19 \\
  & Class-cond. OT  & 5.79 & 5.20 & 5.08 & 5.83 & 5.18 & 5.25 \\
  \rowcolor{aithyra-shade}\cellcolor{white}
  & \gt (from \gt teacher) & \textbf{5.15} & \textbf{4.35} & \textbf{4.15} & \textbf{5.66} & \textbf{4.17} & \textbf{4.14} \\
  \rowcolor{aithyra-shade}\cellcolor{white}\multirow{-4}{*}{DMF-B/2}
  & \gt (from ind.\ teacher) & 5.88 & 5.32 & 5.22 & 5.88 & 4.21 & 4.18 \\
  \midrule
  & Independent            & 5.50 & 5.74 & 2.95 & \textbf{3.31} & 2.82 & 2.68 \\
  \rowcolor{aithyra-shade}\cellcolor{white}\multirow{-2}{*}{DMF-L/2}
  & \gt                    & \textbf{3.94} & \textbf{3.84} & \textbf{2.69} & 4.03 & \textbf{2.46} & \textbf{2.37} \\
  \bottomrule
\end{tabular}%
\hspace{1.5em}%
\raisebox{-0.5\height}{\includegraphics[height=3.6cm]{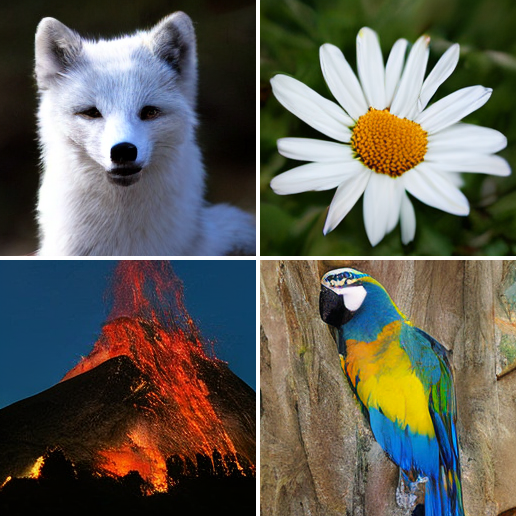}}%
\end{table}

\xhdr{Adaptive step sampling} We compare coupling plans under the adaptive \texttt{dopri5} solver as shown in Table~\ref{tab:sit_b2_nfe}. \gt attains the lowest FID and FD\textsubscript{DINOv2} overall, while semi-discrete \gt requires the fewest function evaluations and outperforms the other couplings at high guidance scales.

\cut{
\begin{table}[h]
\centering
\small
\caption{FID $\downarrow$ and NFE $\downarrow$ for SiT-B/2 over coupling plans and guidance scales $w$.}
\setlength{\tabcolsep}{4pt}
\begin{tabular}{c cc cc cc cc}
\toprule
 & \multicolumn{2}{c}{Independent} & \multicolumn{2}{c}{Class-cond.\ OT} & \multicolumn{2}{c}{\gt} & \multicolumn{2}{c}{\gt+SD-OT} \\
\cmidrule(lr){2-3}\cmidrule(lr){4-5}\cmidrule(lr){6-7}\cmidrule(lr){8-9}
$w$ & FID & NFE & FID & NFE & FID & NFE & FID & NFE \\
\midrule
1.0 & \textbf{25.01} & 50.83 & -- & -- & 29.65 & 49.17 & 42.39 & \textbf{38.00} \\
2.0 &  5.17 & 60.99 & -- & -- & \textbf{4.06} & 61.54 & 7.56 & \textbf{50.00} \\
3.0 & -- & -- & -- & -- & 8.77 & 73.45 & \textbf{4.79} & \textbf{62.80} \\
4.0 & 15.77 & 90.04 & -- & -- & 13.37 & 85.63 & \textbf{5.47} & \textbf{77.09} \\
\bottomrule
\end{tabular}
\label{tab:sit_b2_nfe}
\end{table}}

\newcolumntype{g}{>{\columncolor{aithyra-shade}}c}

\begin{table}[h]
\centering
\small
\caption{FID $\downarrow$, FD\textsubscript{DINOv2} $\downarrow$, and NFE $\downarrow$ for SiT-B/2 over coupling plans and guidance scales $w$.}
\setlength{\tabcolsep}{4pt}
\resizebox{1.0\textwidth}{!}{
\begin{tabular}{c ccc ccc ccc ccc}
\toprule
 & \multicolumn{3}{c}{Independent} & \multicolumn{3}{c}{Class-cond.\ OT} & \multicolumn{3}{c}{\gt} & \multicolumn{3}{c}{\gt+SD-OT} \\
\cmidrule(lr){2-4}\cmidrule(lr){5-7}\cmidrule(lr){8-10}\cmidrule(lr){11-13}
$w$ & FID & FD\textsubscript{DINOv2} & NFE & FID & FD\textsubscript{DINOv2} & NFE & FID & FD\textsubscript{DINOv2} & NFE & FID & FD\textsubscript{DINOv2} & NFE \\
\midrule
1.0 & \textbf{25.01} & \textbf{586.47} & 50.83 & 25.23 & 589.00 & 47.73 & \cellcolor{aithyra-shade}29.65 & \cellcolor{aithyra-shade}618.75 & \cellcolor{aithyra-shade}49.17 & \cellcolor{aithyra-shade}42.39 & \cellcolor{aithyra-shade}772.55 & \cellcolor{aithyra-shade}\textbf{38.00} \\
2.0 &  5.17 & \textbf{206.41} & 60.99 & 5.08 & 208.96 & 60.56 & \cellcolor{aithyra-shade}\textbf{4.06} & \cellcolor{aithyra-shade}220.08 & \cellcolor{aithyra-shade}61.54 & \cellcolor{aithyra-shade}7.56 & \cellcolor{aithyra-shade}341.00 & \cellcolor{aithyra-shade}\textbf{50.00} \\
3.0 & 11.13 & 138.30 & 76.39 & 10.98 & 137.98 & 75.84 & \cellcolor{aithyra-shade}8.77 & \cellcolor{aithyra-shade}130.56 & \cellcolor{aithyra-shade}73.45 & \cellcolor{aithyra-shade}\textbf{4.79} & \cellcolor{aithyra-shade}226.95 & \cellcolor{aithyra-shade}\textbf{62.80} \\
4.0 & 15.77 & 132.94 & 90.04 & 15.56 & 132.06 & 89.86 & \cellcolor{aithyra-shade}13.37 & \cellcolor{aithyra-shade}\textbf{116.34} & \cellcolor{aithyra-shade}85.63 & \cellcolor{aithyra-shade}\textbf{5.47} & \cellcolor{aithyra-shade}194.00 & \cellcolor{aithyra-shade}\textbf{77.09} \\
\bottomrule
\end{tabular}
}
\label{tab:sit_b2_nfe}
\end{table}

\begin{figure}[!ht]
    \centering
\begin{minipage}[b]{0.42\linewidth}
        \centering
        \includegraphics[width=\linewidth]{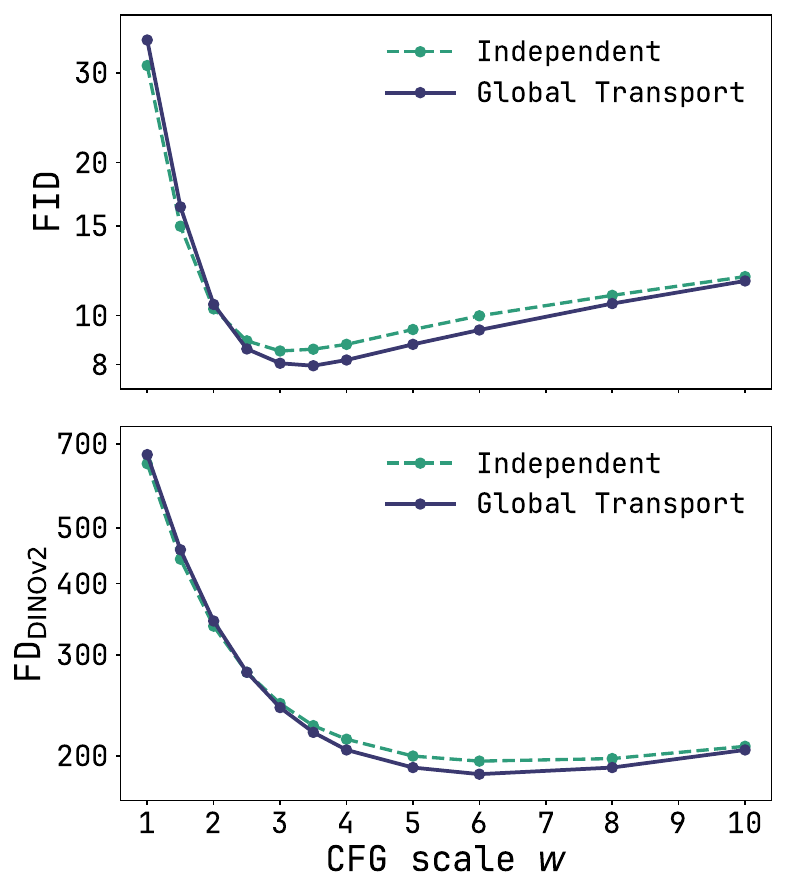}
    \end{minipage}\hspace{0.02\linewidth}%
    \begin{minipage}[b]{0.42\linewidth}
        \centering
        \includegraphics[width=\linewidth]{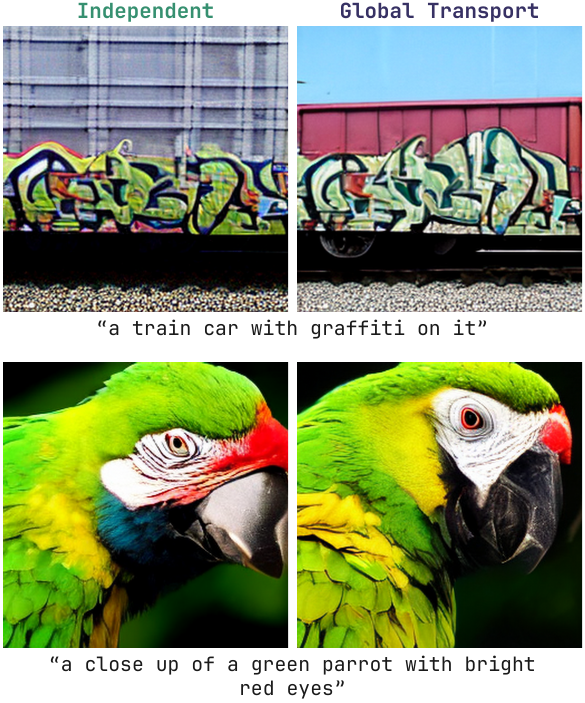}
    \end{minipage}
    \caption{\textbf{Left:} FID$\downarrow$ and FD\textsubscript{DINOv2}$\downarrow$ for text-conditioned ImageNet-256 across guidance scales $w$. \textbf{Right:} Text-conditioned ImageNet samples from independent and \gt coupling models at $w=4$.}
    \label{fig:metrics_cont_cond}
    \label{fig:captioned-imagenet}
\end{figure}

\xhdr{Continuous Conditions} We next demonstrate that \gt improves text-conditioned generation under classifier-free guidance. Following the text-conditioned ImageNet-256 setup of \textcite{cheng2025curse}, we condition on image captions from an enriched version of ImageNet~\parencite{visuallayer2024imagenet1kvl}. Captions are encoded by a frozen pretrained CLIP text encoder followed by an MLP that maps them to the conditioning signal. The \texttt{null} condition used for guidance dropout is a zero vector. As shown in Figure~\ref{fig:metrics_cont_cond}, \gt achieves lower FID and FD\textsubscript{DINOv2} at higher guidance scales.

\begin{table}[h]
\centering
\caption{Comparison of \gt with single-cell generative models on distribution-matching metrics (RBF-kernel MMD and 2-Wasserstein distance), averaged over three seeds.}
\label{tab:main_results}
\setlength{\tabcolsep}{4pt}
\resizebox{0.9\textwidth}{!}{%
\begin{tabular}{lcccccc}
\toprule
& \multicolumn{2}{c}{PBMC3K} & \multicolumn{2}{c}{Dentate gyrus}
& \multicolumn{2}{c}{HLCA} \\
\cmidrule(lr){2-3} \cmidrule(lr){4-5} \cmidrule(lr){6-7} 
& MMD ($\downarrow$) & WD ($\downarrow$)
& MMD ($\downarrow$) & WD ($\downarrow$)
& MMD ($\downarrow$) & WD ($\downarrow$) \\
\midrule
c-CFGen             & 0.45 $\pm$ 0.00 & 11.17 $\pm$ 0.02 & \underline{0.06 $\pm$ 0.00} & 7.26 $\pm$ 0.03 & 
\textbf{0.06 $\pm$ 0.00} & 5.06 $\pm$ 0.00 \\
c-CFGen-linear     & \underline{0.41 $\pm$ 0.00} & \underline{10.27 $\pm$ 0.08} & \underline{0.06 $\pm$ 0.00} & 6.69 $\pm$ 0.01 & 
\underline{0.07 $\pm$ 0.00} & \textbf{4.89 $\pm$ 0.01} \\
scDiffusion & 0.67 $\pm$ 0.06 & 12.11 $\pm$ 0.16 & \underline{0.06 $\pm$ 0.00} & \textbf{5.89 $\pm$ 0.01} & 
0.12 $\pm$ 0.00 & 5.42 $\pm$ 0.01 \\
scVI        & 0.58 $\pm$ 0.01 & 13.39 $\pm$ 0.16 & 0.11 $\pm$ 0.00 & 7.34 $\pm$ 0.03 & 
0.13 $\pm$ 0.00 & 6.39 $\pm$ 0.01 \\
\midrule
\rowcolor{aithyra-shade}\gt \textbf{(ours)}     & \textbf{0.39 $\pm$ 0.00} & \textbf{9.80 $\pm$ 0.02} & \textbf{0.05 $\pm$ 0.00} & \underline{6.61 $\pm$ 0.01} & 
\underline{0.07 $\pm$ 0.00} & \underline{4.98 $\pm$ 0.00} \\
\bottomrule
\end{tabular}%
}
\end{table}

\subsection{Single-cell Experiments} We evaluate \gt on conditional single-cell generation following~\textcite{palma2025multimodal}, conditioning on cell type for PBMC3K\footnote{\url{https://satijalab.org/seurat/articles/pbmc3k_tutorial.html}}, Dentate gyrus~\parencite{la2018rna} and HLCA~\parencite{sikkema2023integrated}. Against c-CFGen~\parencite{palma2025multimodal}, its linear interpolant variant, scDiffusion~\parencite{luo2024scdiffusion} and scVI~\parencite{gayoso2021joint}, \gt is best on three of six metrics and second best on three out of six (Table~\ref{tab:main_results}, details in Appendix~\ref{appendix:single_cell}).

\section{Related Work}
\xhdr{Optimal Transport for Generative Models} Optimal transport~\parencite{benamou2000computational} is widely used to improve unconditional~\parencite{tong_conditional_2023, tong_simulation-free_2023, pmlr-v202-pooladian23a, calvoordonez2026weightedconditionalflowmatching} and class-conditional~\parencite{cheng2025curse, chemseddine2025conditional} generation, at scale via semi-discrete potentials~\parencite{mousavi-hosseini2026flow, kong2026alignflow} and ``re-flow''-style strategies that exploit flow invertibility~\parencite{kim2025simple, berthelot2026coupling}. In biology, it has been applied to single-cell trajectory inference~\parencite{schiebinger_optimal-transport_2019, kapusniak2024metric, petrovic2025curly} and measure-to-measure transport~\parencite{haviv2025wasserstein,vandergrift2026measuretomeasureregressiontransformers}.

\xhdr{Flow Maps} Flow maps~\parencite{frans2025one,geng2026mean, boffi2026build} have recently emerged as an efficient route to one- and few-step generators, either distilled~\parencite{sabour2026align, lee2026decoupled} or trained from scratch~\parencite{geng2026mean, boffi2026build, geng2026improved}, with applications to image~\parencite{lu2026onestep, wang2026transition} and video~\parencite{gu2026anyflow, shaul2026paralleldecodingdistillationfast} generation. \looseness=-1

\xhdr{Classifier-Free Guidance for Flow Models} Classifier-free guidance (CFG)~\parencite{ho2022classifier} is a widely adopted strategy for improving conditional generation with diffusion and flow matching models~\parencite{zheng2023guidedflowsgenerativemodeling}. Subsequent work includes guidance interval tuning~\parencite{kynkaanniemi2024applying}, velocity field projection~\parencite{fan2025cfg, cai2026improving}, guidance weight schedules~\parencite{chung2025cfg++,wang2024analysisclassifierfreeguidanceweight, galashov2026learn} and guiding with a weaker checkpoint of the same model~\parencite{karras2024guiding}.\looseness=-1

\section{Conclusion}
In this work, we studied how coupling choice affects conditional flow models sampled with classifier-free guidance. We introduced \globalt (\gt), which constructs \emph{condition-agnostic} OT couplings and requires no change to the model architecture, sampler, or guidance rule. Although this choice may lead to a mismatch between the source distribution seen by each condition during training and degrades performance in the unguided conditional generation setting, in our experiments \gt improves guided generation across all settings we study. This contrast shows that unguided performance is not a reliable basis for choosing a coupling when CFG is used during inference, and opens up a new direction of inquiry in designing couplings for more exotic inference strategies.

\paragraph{Limitations and future work} Our prediction gap bound does not control the learned fields along CFG trajectories or guarantee a resulting hierarchy of generation quality. Understanding when each construction is preferable and how coupling choice interacts with other guidance and post-training methods are useful directions for future work.  As CFG is primarily used in image and video generation, transferring these gains to domains outside of cells in the life sciences, where other factors dominate, remains open.


\subsection*{Acknowledgments}
The authors would like to thank Romeo Passaro who participated in planting the seeds of this idea, initial experiments and discussions, as well as Scott le Roux for feedback on initial draft. Danyal Rehman received financial support from the Natural Sciences and Engineering Research Council’s (NSERC) Banting Postdoctoral Fellowship under Funding Reference No. 198506. 
Lazar Atanackovic was supported by the Canada CIFAR AI Chairs program.
The research was enabled in part by computational resources provided by AITHYRA (\url{https://aithyra.at}), the Digital Research Alliance of Canada
(\url{https://alliancecan.ca}), the Alberta Machine Intelligence Institute (\url{https://www.amii.ca}), and NVIDIA. 
AITHYRA is supported by the Austrian Academy of Sciences and the not-for-profit Boehringer Ingelheim Stiftung.  
This research is partially supported by EPSRC Turing AI World-Leading Research Fellowship No.
EP/X040062/1 and EPSRC AI Hub on Mathematical Foundations of Intelligence: An ``Erlangen Programme'' for AI No. EP/Y028872/1.

\newrefcontext[sorting=nyt]
\printbibliography[heading=bibintoc]
\newpage
\appendix
\crefalias{section}{appendix}
\crefalias{subsection}{appendix}
\crefalias{subsubsection}{appendix}

\startcontents[]
\printcontents[]{l}{1}[3]{{\bfseries \large Appendices}}

\section{Synthetic Example}
\label{appendix: case_study_toy}
\xhdr{Dataset construction} We construct a 40-Gaussian-Mixture with 4 classes. Following~\cite{midgley2023flow}, we use a two-dimensional 40-component Gaussian mixture (40-GMM). All components have equal weight and share the isotropic covariance
\begin{equation}
    \Sigma = \begin{pmatrix} 40 & 0 \\ 0 & 40 \end{pmatrix},
\end{equation}
with means $\mu_i$ drawn uniformly from the box $[-40,40]^2$, i.e.\ $\mu_i \sim \mathcal{U}(-40,40)^2$, giving the density
\begin{equation}
    p_{\mathrm{gmm}}(x) = \frac{1}{40}\sum_{i=1}^{40} \mathcal{N}(x;\,\mu_i,\Sigma).
\end{equation}

The ground truth is shown in figure \ref{fig:ground_truth_toy}.
\begin{figure}[h]
    \centering
    \includegraphics[width=0.5\linewidth]{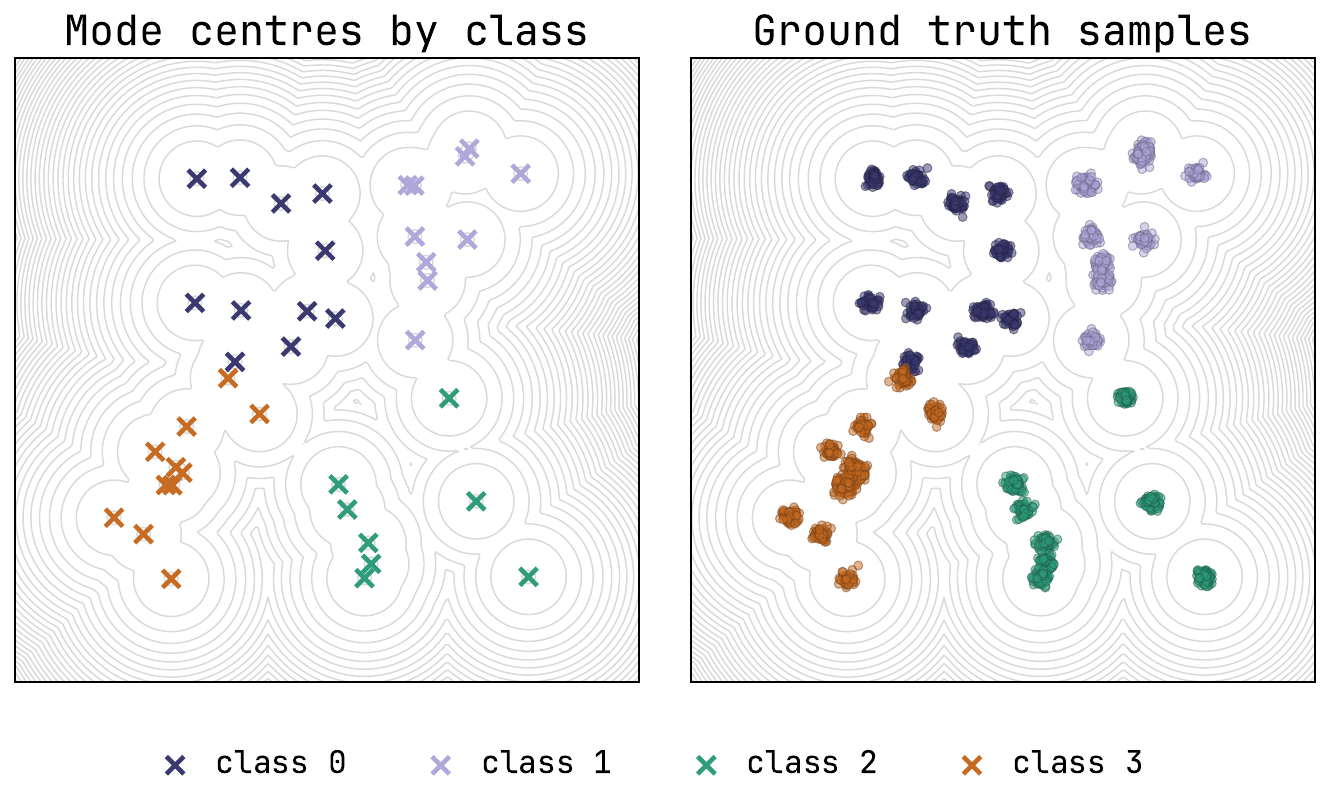}
    \caption{40-GMM ground truth data. Left: Mode centers by class, Right: Ground truth samples.}
    \label{fig:ground_truth_toy}
\end{figure}

\xhdr{Generated trajectories} 
We provide generated trajectories across a wide range of guidance scales for Independent, Class-conditional and \gt transport. We observe that ``yo-yo'' is present in Independent with trajectories curving inwards and outwards close to the conditioned class. 
\begin{figure}[h]
    \centering
    \includegraphics[width=\linewidth]{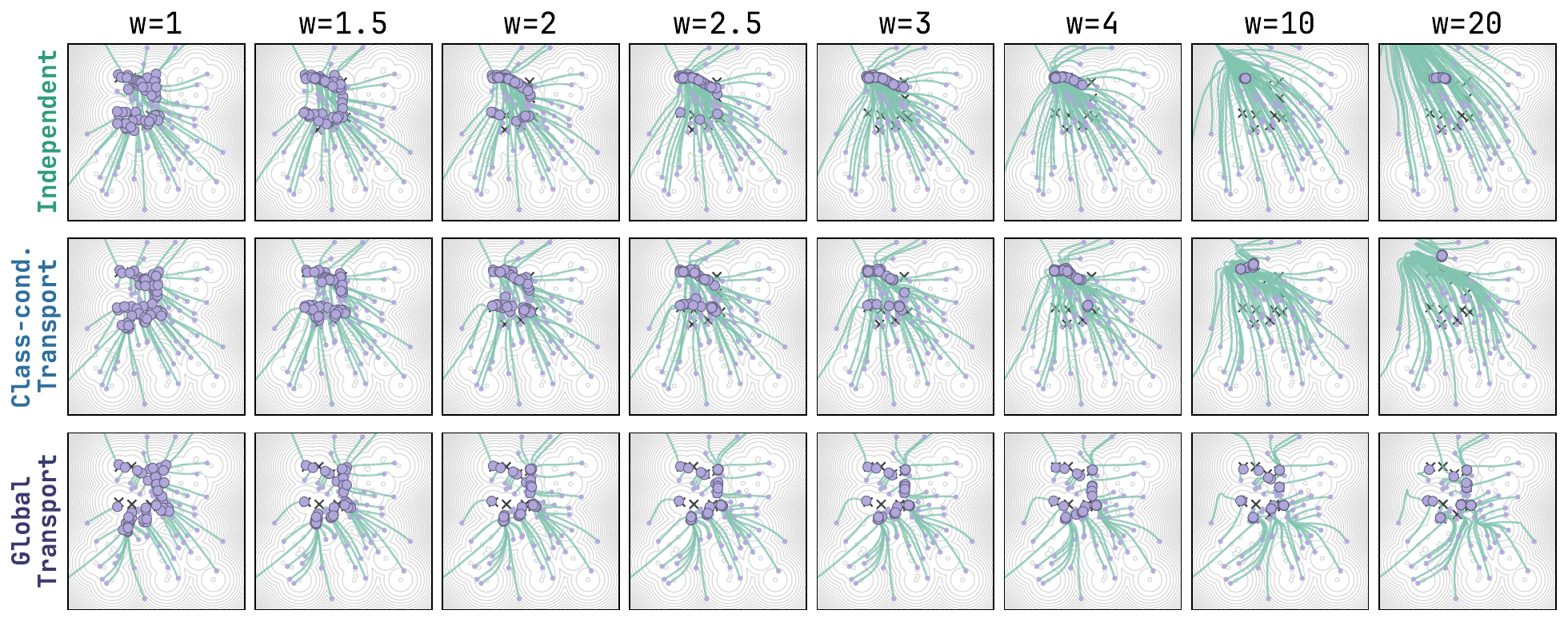}
    \caption{Generated trajectories for Independent, Class-conditional and \gt on 40-GMM.}
    \label{fig:toy_trajs}
\end{figure}

\xhdr{Velocity Field Heatmaps} We compute heatmaps for unconditional and conditional velocity fields comparing different coupling plans. Under the independent coupling, both unconditional and conditional field start as a sink located around class origin, reorganizing rapidly at $t=0.5$. Under \gt we observe that unconditional velocity field forms already at $t=0$ and remains stable across the trajectory.\looseness=-1 
\begin{figure}[h]
    \centering
    \includegraphics[width=0.8\linewidth]{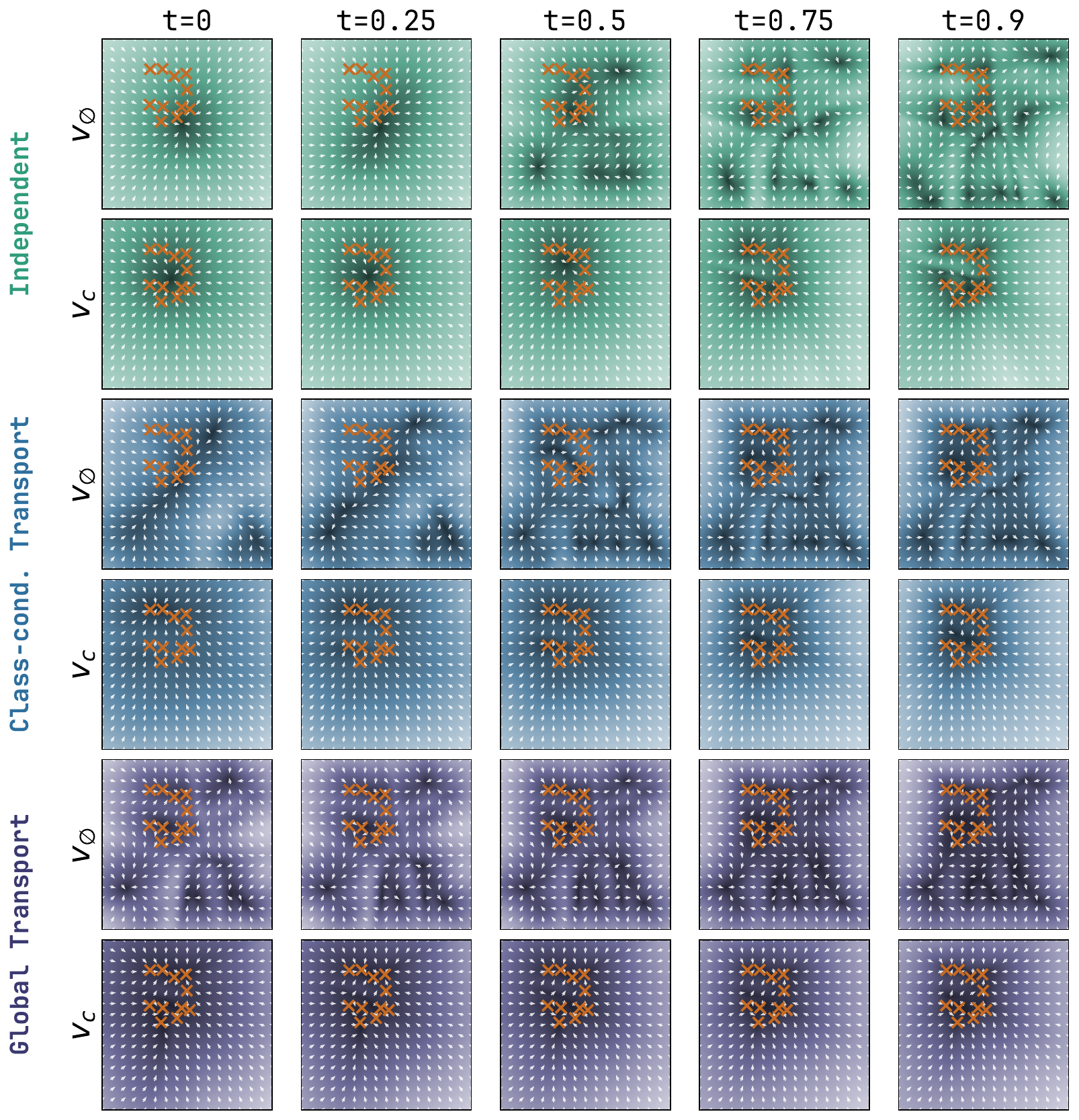}
    \caption{Velocity field heat maps for conditional $v_c$ and unconditional $v_{\varnothing}$ for Independent, Class-conditional and \gt transport.}
    \label{fig:toy_fields}
\end{figure}

\xhdr{Additional samples} We provide additional generated samples to accompany figure \ref{fig:toy-guidance} demonstrating \gt does not drop off the modes even in cases of extremely high guidance scale $w=20.$
\begin{figure}[h]
    \centering
    \includegraphics[width=\linewidth]{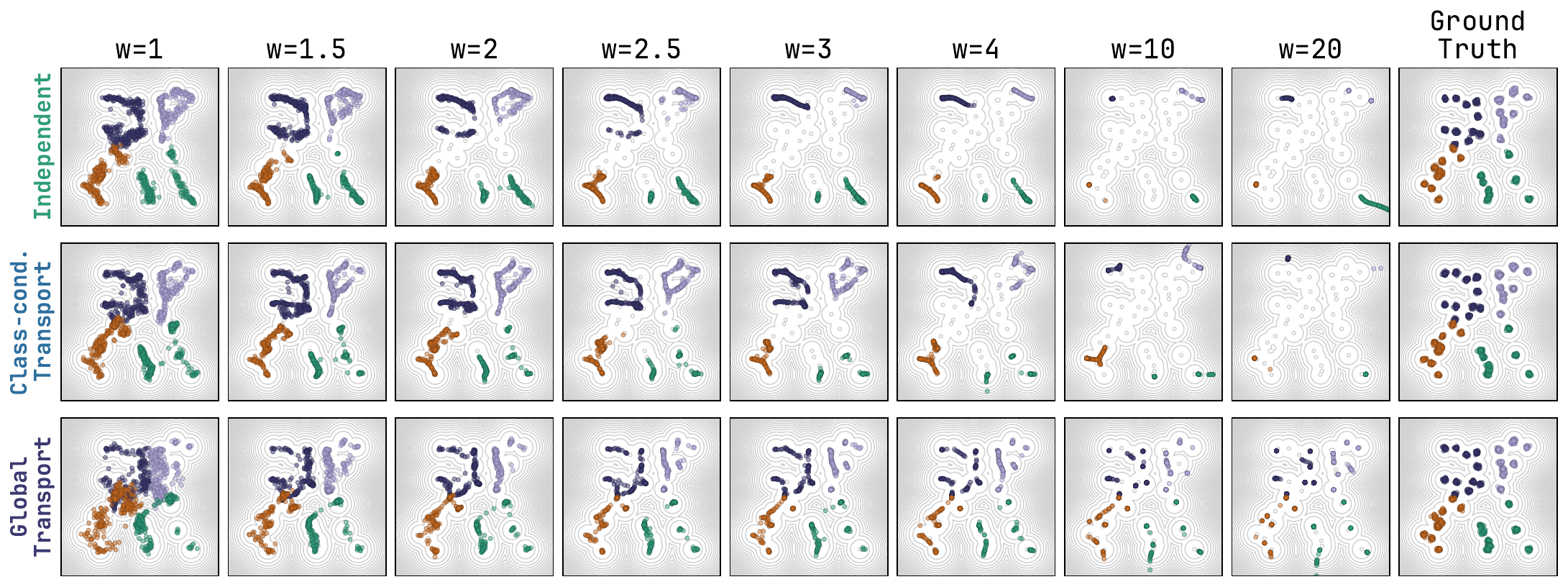}
    \caption{Additional 40-GMM samples across a range of guidance scales.}
    \label{fig:placeholder}
\end{figure}
\section{Proofs and Derivations}
\subsection{Relationship between coupling and prediction gap}
\label{appendix: g_proof}
For the ease of notation we will write $u_t(x_t)$ as $v_u$ and $u_t(x_t\mid c)$ as $v_c$.

\begin{proposition}
    For any coupling $\pi$,
\(
    \:\:\mathbb{E}_\pi\left[
        \left\|
        u_t(x_t\mid c)-u_t(x_t)
        \right\|^2
    \right]
    =
    L_u^\pi(t)-L_c^\pi(t).
\)
\end{proposition}

\begin{proof}    
We can write the following by adding and subtracting $u$ to $g = v_c - v_u$
\begin{equation}
    u - v_u = (u - v_c) + (v_c-v_u)
\end{equation}
By squaring and taking the norm we get
\begin{align}
    \|u-v_u\|^2
    &=
    \|u-v_c\|^2
    +
    \|v_c-v_u\|^2
    \nonumber\\
    &\quad+
    2\langle u-v_c,v_c-v_u\rangle.
    \label{eq:proof_expansion}
\end{align}
We now take expectation over $(x_0, x_1, c)\sim \pi$ 
\begin{align}
    \mathbb{E}_{\pi}\|u-v_u\|^2
    &=
    \mathbb{E}_{\pi}\|u-v_c\|^2
    +
    \mathbb{E}_{\pi}\|v_c-v_u\|^2
    \nonumber\\
    &\quad+
    2\mathbb{E}_{\pi}\langle u-v_c,v_c-v_u\rangle.
    \label{eq:proof_expansion_exp}
\end{align}
The expectation of the final term is zero:
\begin{align}
&\mathbb{E}_\pi
\left[
    \langle u-v_c,v_c-v_u\rangle
\right]
\nonumber\\
&=
\mathbb{E}_\pi
\left[
    \mathbb{E}_\pi
    \left[
        \langle u-v_c,v_c-v_u\rangle
        \mid x_t,c
    \right]
\right]
\nonumber\\
&=
\mathbb{E}_\pi
\left[
    \left\langle
        \mathbb{E}_\pi[u-v_c\mid x_t,c],
        v_c-v_u
    \right\rangle
\right].
\end{align}
By the definition of the conditional field,
\begin{equation}
    v_c(x_t,t,c)
    =
    \mathbb{E}_\pi[u\mid x_t,c],
\end{equation}
and therefore
\begin{equation}
    \mathbb{E}_\pi[u-v_c\mid x_t,c]=0.
\end{equation}
Hence
\begin{equation}
    \mathbb{E}_\pi
    \left[
        \langle u-v_c,v_c-v_u\rangle
    \right]
    =0.
\end{equation}

Taking expectations in \eqref{eq:proof_expansion} now gives
\begin{align}
    \mathbb{E}_\pi\|u-v_u\|^2
    &=
    \mathbb{E}_\pi\|u-v_c\|^2
    +
    \mathbb{E}_\pi\|v_c-v_u\|^2.
\end{align}
Using the definitions of \(L_u^\pi(t)\) and \(L_c^\pi(t)\), we obtain
\begin{equation}
    \boxed{\mathbb{E}_\pi\|v_c-v_u\|^2
    =
    L_u^\pi(t)-L_c^\pi(t)}
\end{equation}
which proves the result.
\end{proof}

\begin{proposition}
    The error term $L_u^\pi(t)$ can be upper bounded as
    \vspace{-1em}
    \begin{equation}
    \label{eq:bound}
        \textstyle\int_0^1L_u^\pi(t)\:\mathrm{d}t\leq c(\pi)-\mathcal{W}_2^2(p_0,p_1) =: d(\pi),
    \end{equation}
    where $c(\pi)$ is the coupling cost and $d(\pi)\geq0$ denotes the excess quadratic transport cost, i.e., the suboptimality of $\pi$ relative to the population-optimal coupling~\parencite{boite2026expected}.
    \end{proposition}
    \begin{proof}
        
    Let us consider coupling cost $c(\pi)=\mathbb{E}_\pi\left[\left\|x_1-x_0\right\|^2\right]$ and ideal minimizer of the unconditional field $ v_u(x,t)=\mathbb{E}_{\pi}\left[u\,\middle|\,x=x_t\right]$, where $u = x_1-x_0$. We assume that the joint distribution $\pi(x_0,x_1)$ induces a probability path $x_t\sim p_t$. We can then express time-dependent error $L_u^\pi(t)$ as
\begin{align}
    L_u^\pi(t)
    &=
    \mathbb{E}_{\pi}
    \left[
        \left\|
            u-v_u(x_t,t)
        \right\|^2
    \right]
    \\
    &=
    \mathbb{E}_\pi
        \left[\left\|u\right\|^2\right] -2\underbrace{\mathbb{E}_\pi\langle u,v_u(x_t,t)\rangle}_{\star} + \mathbb{E}_{x_t\sim p_t}
        \left[\left\|v_u(x_t,t)\right\|^2\right] 
\end{align}
We further express $\star$ as 
\begin{align}
    \mathbb{E}_\pi\big[\langle u,\,v_u(x_t,t)\rangle\big]
    &= \mathbb{E}_{x_t\sim p_t}
       \Big[\big\langle
            \underbrace{\mathbb{E}_\pi\!\left[u\mid x_t\right]}_{=\;v_u(x_t,t)},\;
            v_u(x_t,t)
        \big\rangle\Big] \notag\\[4pt]
    &= \mathbb{E}_{x_t\sim p_t}\big[\langle v_u(x_t,t),\,v_u(x_t,t)\rangle\big] \notag\\[2pt]
    &= \mathbb{E}_{x_t\sim p_t}\big[\left\|v_u(x_t,t)\right\|^2\big].
\end{align}
From here it follows
\begin{equation}
    L_u^\pi(t)=\underbrace{\mathbb{E}_\pi
        \left[\left\|u\right\|^2\right]}_{c(\pi)} - \mathbb{E}_{x_t\sim p_t}\left[\left\|v_u(x_t,t)\right\|^2\right]
\end{equation}
or equivalently
\begin{equation}
    L_u^\pi(t)=c(\pi) - \mathbb{E}_{x_t\sim p_t}\left[\left\|v_u(x_t,t)\right\|^2\right]
\end{equation}
If we integrate both sides with respect to $t$
\begin{equation}
    \int_0^1L_u^\pi(t)\mathrm{d}t=c(\pi) - \int_0^1\mathbb{E}_{x_t\sim p_t}\left[\left\|v_u(x_t,t)\right\|^2\right]\mathrm{d}t
\end{equation}
Using \textcite{benamou2000computational} we can express
\begin{equation}
    \int_0^1\mathbb{E}_{x_t\sim p_t}\left[\left\|v_u(x_t,t)\right\|^2\right]\mathrm{d}t \geq \mathcal{W}_2^2(p_0,p_1)
\end{equation}
From which identity follows
 \begin{equation}
        \boxed{\int_0^1L_u^\pi(t)\mathrm{d}t\leq \underbrace{c(\pi)-\mathcal{W}_2^2(p_0,p_1)}_{d(\pi)}}.
    \end{equation}
    \end{proof}
\cut{\subsection{Expressing $L_u^\pi$ and $L_c^\pi$ as conditional variance}
We can further identify that error terms $L^\pi_u(t)$ and  $L^\pi_c(t)$ can be expressed as conditional variance following the law of total variance
\begin{align}
    L_u^\pi(t)
    &=
    \mathbb{E}_{\pi}
    \left[
        \left\|
            u-v_u(x_t,t)
        \right\|^2
    \right]
    \\
    &=
    \mathbb{E}_{x_t\sim p_t}
    \left[
        \mathbb{E}_{\pi}
        \left[
            \left\|
                u-\mathbb{E}_{\pi}[U\mid x_t]
            \right\|^2
            \,\middle|\,
            x_t
        \right]
    \right]
    \\
    &=
    \mathbb{E}_{x_t\sim p_t}
    \left[
        \operatorname{Var}_{\pi}(u\mid x_t)
    \right].
\end{align}
Following above we can also write $L_c^\pi(t) = \mathbb{E}_{x_t\sim p_t}
    \left[\operatorname{Var}_{\pi}(u\mid x_t,c)\right]$
}

\subsection{Proof of Proposition 4.2: Hierarchy of Coupling Costs}
\label{appendix:proof_coupling_order}

\begin{proposition}[Coupling Cost Ordering]
For all $\beta \ge 0$, the quadratic transport costs of the global transport coupling $\pi_{\gt}^\star$, the class-conditional optimal transport coupling $\pi_{\mathrm{C^2OT}(\beta)}^\star$, and the independent coupling $\pi_{\mathrm{ind}} = p_0 \otimes p_1$ satisfy:
\begin{equation}
    \operatorname{Cost}(\pi_{\gt}^\star)
    \le
    \operatorname{Cost}(\pi_{\mathrm{C^2OT}(\beta)}^\star)
    \le
    \operatorname{Cost}(\pi_{\mathrm{ind}}),
\end{equation}
with the lower equality achieved when $\beta = 0$.
\end{proposition}

\begin{proof}
Let $\mathcal{X} = \mathbb{R}^d$ denote the state space and $\mathcal{C}$ denote the conditioning space (endowed with metric $d_{\mathcal{C}}$ or squared norm $\|c_0 - c_1\|^2$). 
Let $p_0$ be the source noise distribution on $\mathcal{X}$, and let $p_1(x_1, c_1)$ be the joint data-conditioning distribution on $\mathcal{X} \times \mathcal{C}$ with spatial marginal $p_1(x_1) = \int_{\mathcal{C}} p_1(x_1, c_1) \, \mathrm{d}c_1$.
Recall that the quadratic transport cost of any coupling $\pi$ with marginals $p_0$ and $p_1(x_1)$ is defined by:
\begin{equation}
    \operatorname{Cost}(\pi) := \mathbb{E}_{(x_0, x_1) \sim \pi} \left[ \|x_1 - x_0\|^2 \right].
\end{equation}

\paragraph{Part 1: $\operatorname{Cost}(\pi_{\gt}^\star) \le \operatorname{Cost}(\pi_{\mathrm{C^2OT}(\beta)}^\star)$.}
By definition, the global transport coupling $\pi_{\gt}^\star$ is the minimizer of the unconstrained Kantorovich optimal transport problem between $p_0$ and the spatial data marginal $p_1(x_1)$:
\begin{equation}
    \label{eq:ot_primal}
    \pi_{\gt}^\star \in \operatorname*{argmin}_{\pi \in \Pi(p_0, p_1)} \mathbb{E}_{\pi}\left[ \|x_1 - x_0\|^2 \right].
\end{equation}
Therefore, for \emph{any} admissible joint distribution $\pi \in \Pi(p_0, p_1)$, we have by definition of the infimum:
\begin{equation}
    \label{eq:global_infimum}
    \operatorname{Cost}(\pi_{\gt}^\star) \le \operatorname{Cost}(\pi).
\end{equation}

Now consider the class-conditional optimal transport coupling $\pi_{\mathrm{C^2OT}(\beta)}^\star$. For any regularization strength $\beta \ge 0$, $\pi_{\mathrm{C^2OT}(\beta)}^\star$ is defined as the minimizer of the joint spatial and condition transport problem:
\begin{equation}
    \label{eq:c2ot_objective}
    \pi_{\mathrm{C^2OT}(\beta)}^\star \in \operatorname*{argmin}_{\pi \in \Pi(p_0, p_1)} \mathbb{E}_{((x_0, c_0), (x_1, c_1)) \sim \pi} \left[ \|x_1 - x_0\|^2 + \beta \|c_1 - c_0\|^2 \right].
\end{equation}
Since the spatial marginals of any candidate coupling in~\eqref{eq:c2ot_objective} are constrained to be $p_0$ and $p_1(x_1)$, the resulting optimal coupling $\pi_{\mathrm{C^2OT}(\beta)}^\star$ is itself an admissible coupling in $\Pi(p_0, p_1)$. 

Applying the optimality property of $\pi_{\gt}^\star$ from~\eqref{eq:global_infimum} directly to $\pi = \pi_{\mathrm{C^2OT}(\beta)}^\star$ yields:
\begin{equation}
    \operatorname{Cost}(\pi_{\gt}^\star) \le \operatorname{Cost}(\pi_{\mathrm{C^2OT}(\beta)}^\star), \quad \forall \beta \ge 0.
\end{equation}
When $\beta = 0$, the penalty on the condition labels vanishes, and the objective in~\eqref{eq:c2ot_objective} reduces identically to the objective in~\eqref{eq:ot_primal}. Thus, $\pi_{\mathrm{C^2OT}(0)}^\star = \pi_{\gt}^\star$, achieving equality.

\paragraph{Part 2: $\operatorname{Cost}(\pi_{\mathrm{C^2OT}(\beta)}^\star) \le \operatorname{Cost}(\pi_{\mathrm{ind}})$.}
Next, we show that the independent product coupling $\pi_{\mathrm{ind}} = p_0 \otimes p_1$ serves as an upper bound on $\operatorname{Cost}(\pi_{\mathrm{C^2OT}(\beta)}^\star)$.
Let $(x_0, c_0) \sim p_0 \times p_{\mathcal{C}}$ and $(x_1, c_1) \sim p_1$ be drawn independently, where $p_{\mathcal{C}}$ is the marginal condition distribution under $p_1$ assigned independently to noise points $x_0$. 
Because $\pi_{\mathrm{ind}} \in \Pi(p_0, p_1)$ is an admissible joint distribution, the minimality of $\pi_{\mathrm{C^2OT}(\beta)}^\star$ for the joint objective in~\eqref{eq:c2ot_objective} implies:
\begin{align}
    \mathbb{E}_{\pi_{\mathrm{C^2OT}(\beta)}^\star} \left[ \|x_1 - x_0\|^2 \right] 
    &\le \mathbb{E}_{\pi_{\mathrm{C^2OT}(\beta)}^\star} \left[ \|x_1 - x_0\|^2 \right] + \beta\, \mathbb{E}_{\pi_{\mathrm{C^2OT}(\beta)}^\star} \left[ \|c_1 - c_0\|^2 \right] \nonumber \\
    &\le \mathbb{E}_{\pi_{\mathrm{ind}}} \left[ \|x_1 - x_0\|^2 + \beta \|c_1 - c_0\|^2 \right]. \label{eq:c2ot_vs_ind_pre}
\end{align}
In the independent coupling $\pi_{\mathrm{ind}}$, the conditioning assignments $c_0$ and $c_1$ are mutually independent identically distributed samples from $p_{\mathcal{C}}$. 

In the hard class-matching limit ($\beta \to \infty$ or exact class-conditional coupling where $c_0 = c_1$), the transport problem decomposes into $|\mathcal{C}|$ independent within-class sub-problems. For each class $c \in \mathcal{C}$ with prevalence $p(c)$:
\begin{equation}
    \pi_{\mathrm{C^2OT}(\infty)}^\star = \sum_{c \in \mathcal{C}} p(c)\, \pi_c^\star, \quad \text{where } \pi_c^\star \in \operatorname*{argmin}_{\pi \in \Pi(p_0, p_1(\cdot \mid c))} \mathbb{E}_{\pi} \left[ \|x_1 - x_0\|^2 \right].
\end{equation}
For every class $c$, the product measure $\pi_{\mathrm{ind}, c} = p_0 \otimes p_1(\cdot \mid c)$ is an admissible coupling for the $c$-th sub-problem. By the optimality of $\pi_c^\star$ over $\Pi(p_0, p_1(\cdot \mid c))$:
\begin{equation}
    \mathbb{E}_{\pi_c^\star}\left[ \|x_1 - x_0\|^2 \right] \le \mathbb{E}_{\pi_{\mathrm{ind}, c}}\left[ \|x_1 - x_0\|^2 \right].
\end{equation}
Taking the expectation over the class distribution $p(c)$ yields:
\begin{align}
    \operatorname{Cost}(\pi_{\mathrm{C^2OT}(\infty)}^\star) 
    &= \sum_{c \in \mathcal{C}} p(c)\, \mathbb{E}_{\pi_c^\star} \left[ \|x_1 - x_0\|^2 \right] \nonumber \\
    &\le \sum_{c \in \mathcal{C}} p(c)\, \mathbb{E}_{\pi_{\mathrm{ind}, c}} \left[ \|x_1 - x_0\|^2 \right] \nonumber \\
    &= \mathbb{E}_{\pi_{\mathrm{ind}}} \left[ \|x_1 - x_0\|^2 \right] = \operatorname{Cost}(\pi_{\mathrm{ind}}).
\end{align}
By monotonicity of the optimal transport objective under relaxation of the condition constraint, $\operatorname{Cost}(\pi_{\mathrm{C^2OT}(\beta)}^\star) \le \operatorname{Cost}(\pi_{\mathrm{ind}})$ holds for all $\beta \ge 0$.

Combining Part 1 and Part 2 gives the complete chain of inequalities:
\begin{equation}
    \operatorname{Cost}(\pi_{\gt}^\star) \le \operatorname{Cost}(\pi_{\mathrm{C^2OT}(\beta)}^\star) \le \operatorname{Cost}(\pi_{\mathrm{ind}}),
\end{equation}
which concludes the proof.
\end{proof}

\subsection{Coupling controls flow curvature}
We next compare the acceleration in the flow induced by the coupling in the standardized Gaussian setting.
We use a Gaussian plug-in model of mini-batch OT which provides an analytically tractable proxy for the intractable discrete mini-batch coupling used in our experiments.
For a velocity field $v_t$, let $a_t[v](x):=\partial_t v_t(x)+D_xv_t(x)v_t(x)$ be the acceleration experienced by a particle $\dot{x}_t = v_t(x_t)$ following the flow, also known as the \emph{material acceleration}.

\begin{proposition}[Material acceleration under Gaussian couplings]
\label{prop:gaussian-material-acceleration-main}
Let $p_0=p_1=\Nc(0,\Id_d)$, and let $v_t^{\mathrm{ind}}$, $v_{t,n}$, and $v_t^\star$ denote respectively the vector fields constructed using (i) the independent coupling, (ii) mini-batch plug-in Gaussian OT coupling with batch size $n$, and (iii) the population OT coupling. Fixing $d$, as $n\to\infty$, the following holds pointwise in $(t, x)$:
\begin{equation}
  \begin{aligned}
    a_t[v^{\mathrm{ind}}](x) =\tfrac{x}{((1-t)^2+t^2)^2}, \qquad
    a_t[v_{n}](x) =\tfrac{2\lVert x\rVert^2-(d-1)}{2n}x+o\left(\tfrac{1}{n}\right), \quad a_t[v^\star](x)&=0.
    \end{aligned}
    \label{eq:gaussian-material-acceleration-main}
\end{equation}
Define the total material acceleration
\(\mathfrak{A}(v)^2
    := \mathbb{E}_{X_0}\int_0^1
      \lVert a_t[v](X_t)\rVert^2\,\mathrm{d}t
\). 
Then, 
\begin{equation}
    \mathfrak{A}(v^{\mathrm{ind}})
    =\sqrt{\left(2+\tfrac{3\pi}{4}\right)d},
    \qquad
    \mathfrak{A}(v_{n}) =\tfrac{\sqrt{d(d^2+18d+41)}}{2n}+o(n^{-1}),
    \qquad
    \mathfrak{A}(v^\star)=0.
    \label{eq:gaussian-total-material-acceleration-main}
\end{equation}
\end{proposition}
Thus both pointwise and total material acceleration scale as $O(n^{-1})$
for Gaussian plug-in mini-batch OT, remain $O(1)$ for independent FM, and
vanish for population OT. 
\input{gaussian_minibatch_appendix}

\newpage
\section{ImageNet-256}
\label{appendix: imagenet}

\subsection{Implementation Details} 
\xhdr{Training details} In tables \ref{tab:sit_config} and \ref{tab:optim_config}, we show training configurations across our SiT~\parencite{Ma2024SiTEF} and DMF~\parencite{lee2026decoupled} experiments for B/2, L/2 and XL/2 model scales. All ImageNet-256 experiments use mini-batch setting of 256 batch size. DMF~\parencite{lee2026decoupled} conditions the encoder on time $t$ and adds additional time $r$ to condition its decoder, turning a flow matching model into a flow map. Following their set-up, we also select \texttt{logit}-normal distribution to sample $(t,r)$ pairs using the time proposal parameters in table \ref{tab:sit_config}. DMF is trained by finetuning a pretrained SiT model via self-distillation meanflow objective~\parencite{geng2026mean} with a diagonal split of 0.5.
\begin{table}[ht]
    \centering
    \caption{ImageNet-256 implementation details across model scales.}
    \label{tab:sit_config}
    \small
    \begin{tabular}{lccc}
        \toprule
        & B/2 & L/2 & XL/2 \\
        \midrule
        \multicolumn{4}{l}{\textit{Backbone}} \\
        \midrule
        Resolution                 & $256\times256$ & $256\times256$ & $256\times256$ \\
        Params (M)                 & 130            & 458            & 675            \\
        FLOPS (G)                  & 23.1           & 80.7           & 118.6          \\
        Hidden dim.                & 768            & 1024           & 1152           \\
        Heads                      & 12             & 16             & 16             \\
        Patch size                 & $2\times2$     & $2\times2$     & $2\times2$     \\
        Sequence length            & 256            & 256            & 256            \\
        Layers                     & 12             & 24             & 28             \\
        \midrule
        \multicolumn{4}{l}{\textit{Flow Matching (SiT)~\parencite{Ma2024SiTEF}}} \\
        \midrule
        Training iterations        & 800K           & 800K           & 800K           \\
        Epochs                     & 160            & 160            & 160            \\
        Class dropout probability  & 0.1            & 0.1            & 0.1            \\
        \midrule
        \multicolumn{4}{l}{\textit{Flow Map (DMF)~\parencite{lee2026decoupled}}} \\
        \midrule
        DMF depth                  & 8              & 18             & 20             \\
        Training iterations        & 400K           & 400K           & 400K           \\
        Epochs                     & 80             & 80             & 80             \\
        Class dropout probability  & \multicolumn{3}{c}{0.1} \\
        Time proposal $\mu_{\mathrm{FM}}$ & \multicolumn{3}{c}{0.0} \\
        Time proposal $(\mu^{(1)}_{\mathrm{MF}}, \mu^{(2)}_{\mathrm{MF}})$ & \multicolumn{3}{c}{$(0.4, -1.2)$} \\
        Model guidance scale $\omega$ & 0.5         & 0.6            & 0.6            \\
        Guidance interval          & $[0.0, 1.0]$   & $[0.0, 0.7]$   & $[0.0, 0.7]$   \\
        \bottomrule
    \end{tabular}
\end{table}

\begin{table}[ht]
    \centering
    \caption{Optimization hyperparameters, shared across all model scales and both training stages.}
    \label{tab:optim_config}
    \small
    \begin{tabular}{lc}
        \toprule
        Optimizer                 & AdamW \\
        Batch size                & 256 \\
        Learning rate             & 1e-4 \\
        Adam $(\beta_1, \beta_2)$ & $(0.9, 0.95)$ \\
        Adam $\epsilon$           & 1e-8 \\
        Weight decay              & 0.0 \\
        EMA decay rate            & 0.9999 \\
        \bottomrule
    \end{tabular}
\end{table}
\xhdr{Compute budget} In table \ref{tab:compute_budget}, we show cost per optimizer step for each of the coupling plans. We observe that class-conditional OT and \gt add a modest cost per optimizer step, with \gt being slightly higher than class-conditional due to not being restricted per class. 
\begin{table}[h]
\centering
\small
\caption{Compute cost per optimizer step of SiT-B/2 training on ImageNet-256 for the three coupling plans. The second column is the increase over the independent coupling.}
\label{tab:compute_budget}
\begin{tabular}{l cc}
\toprule
Coupling & ms / step & $\uparrow$ cost \\
\midrule
Independent & 155.5 & -- \\
Class-cond.\ OT & 157.9 & $\uparrow$ 1.6\% \\
\gt & 160.4 & $\uparrow$ 3.1\% \\
\bottomrule
\end{tabular}
\end{table}

\xhdr{Choice of $\beta$ for class-conditional OT} We choose $\beta$ such that it prevents any cross-class coupling, thereby computing transport plan within each class following \parencite{cheng2025curse}. In practice, we define $\beta=\beta_m\cdot\frac{1}{B^2}\sum_{i=1}^{B}\sum_{j=1}^{B}\lVert x_i - x_j\rVert_2^2$. We sweep $\beta_m$ and show analysis in figure \ref{fig:ccotbeta}, choosing $\beta_m=10$.

\begin{figure}
    \centering
    \includegraphics[width=0.4\linewidth]{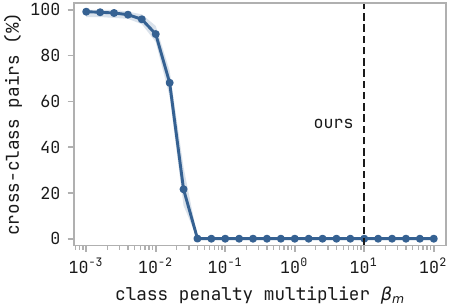}
    \caption{Sweeping $\beta_m$ for class-conditional OT.}
    \label{fig:ccotbeta}
\end{figure}

\xhdr{Text-conditioning for ImageNet} Every ImageNet-1k image is paired with the
captions taken from \textcite{visuallayer2024imagenet1kvl} dataset. Each caption is encoded once with
the frozen DFN5B CLIP ViT-H/14 text tower into its pooled, $\ell_2$-normalised 1024-d text embedding $c\in\mathbb{R}^{1024}$.

\begin{table*}[t]
\caption{\small FID\,$\downarrow$ and FD\textsubscript{DINOv2}\,$\downarrow$ vs.\ guidance scale $w$ for independent, class-conditional OT and \gt couplings across SiT scales (64 Euler steps), under the full guidance interval $[0,1]$.}
\label{tab:sit_full_interval_euler64}
\centering
\scriptsize
\setlength{\tabcolsep}{3.5pt}
\begin{tabular}{ll ccccccccccc}
\toprule
\multicolumn{13}{c}{SiT-B/2} \\
\midrule
Coupling & Metric & $w{=}1.0$ & 1.5 & 2.0 & 2.5 & 3.0 & 3.5 & 4.0 & 5.0 & 6.0 & 8.0 & 10.0 \\
\midrule
\multirow{2}{*}{Independent} & FID & 26.44 & 6.68 & 5.16 & 7.93 & 11.04 & 13.65 & 15.76 & 18.63 & 20.42 & 22.40 & 23.28 \\
 & FD\textsubscript{DINOv2} & 605.11 & 337.16 & 212.66 & 159.87 & 139.10 & 132.46 & 132.27 & 139.64 & 149.79 & 169.51 & 189.02 \\
\midrule
\multirow{2}{*}{Class-cond.\ OT} & FID & 26.72 & 6.70 & 5.08 & 7.81 & 10.90 & 13.47 & 15.59 & 18.44 & 20.28 & 22.25 & 23.12 \\
 & FD\textsubscript{DINOv2} & 607.86 & 340.46 & 215.04 & 160.77 & 139.17 & 131.64 & 131.41 & 138.27 & 148.37 & 168.91 & 189.17 \\
\midrule
\multirow{2}{*}{\gt} & FID & 31.16 & 8.13 & 4.15 & 5.88 & 8.72 & 11.32 & 13.41 & 16.51 & 18.62 & 20.81 & 21.62 \\
 & FD\textsubscript{DINOv2} & 636.79 & 365.26 & 226.75 & 161.94 & 132.62 & 120.25 & 116.36 & 119.84 & 128.33 & 148.82 & 171.99 \\
\bottomrule
\end{tabular}

\vspace{8pt}
\begin{tabular}{ll cccccccc}
\toprule
\multicolumn{10}{c}{SiT-L/2} \\
\midrule
Coupling & Metric & $w{=}1.0$ & 1.5 & 1.75 & 2.0 & 2.5 & 3.0 & 3.5 & 4.0 \\
\midrule
\multirow{2}{*}{Independent} & FID & 13.85 & 2.95 & 3.86 & 5.85 & 10.10 & 13.53 & 16.03 & 17.89 \\
 & FD\textsubscript{DINOv2} & 360.44 & 154.77 & 115.56 & 95.16 & 83.28 & 86.46 & 94.58 & 103.47 \\
\midrule
\multirow{2}{*}{Class-cond.\ OT} & FID & 13.96 & 2.98 & 3.84 & 5.81 & 10.02 & 13.42 & 15.93 & 17.79 \\
 & FD\textsubscript{DINOv2} & 360.38 & 154.83 & 115.16 & 95.01 & 83.41 & 86.38 & 94.01 & 102.64 \\
\midrule
\multirow{2}{*}{\gt} & FID & 17.72 & 2.96 & 2.69 & 3.88 & 7.38 & 10.66 & 13.25 & 15.17 \\
 & FD\textsubscript{DINOv2} & 392.20 & 170.86 & 123.70 & 97.58 & 76.33 & 73.48 & 77.32 & 83.28 \\
\bottomrule
\end{tabular}

\vspace{8pt}
\begin{tabular}{ll ccccccccc}
\toprule
\multicolumn{11}{c}{SiT-XL/2} \\
\midrule
Coupling & Metric & $w{=}1.0$ & 1.5 & 1.75 & 2.0 & 2.25 & 2.5 & 3.0 & 3.5 & 4.0 \\
\midrule
\multirow{2}{*}{Independent} & FID & 12.58 & 2.78 & 3.98 & 6.10 & 8.32 & 10.41 & 13.76 & 16.14 & 17.87 \\
 & FD\textsubscript{DINOv2} & 328.83 & 135.91 & 101.78 & 85.35 & 78.40 & 77.76 & 83.29 & 92.02 & 101.67 \\
\midrule
\multirow{2}{*}{\gt} & FID & 16.11 & 2.63 & 2.59 & 3.87 & 5.65 & 7.47 & 10.70 & 13.16 & 15.09 \\
 & FD\textsubscript{DINOv2} & 357.67 & 150.28 & 108.56 & 85.80 & 73.95 & 69.01 & 68.44 & 73.27 & 80.05 \\
\bottomrule
\end{tabular}
\end{table*}

\subsection{Additional results on ImageNet-256} 
\xhdr{\gt preserves mode stability} Figure~\ref{fig:mode_stability} sweeps $w$ for fixed noise and class. Under the independent coupling, increasing guidance results in abrupt change in composition and pose, e.g.\ from a portrait to a full-body view. Under \gt the object remains stable for the full range of $w$. 
\begin{figure}[t]
    \centering
    \includegraphics[width=\linewidth]{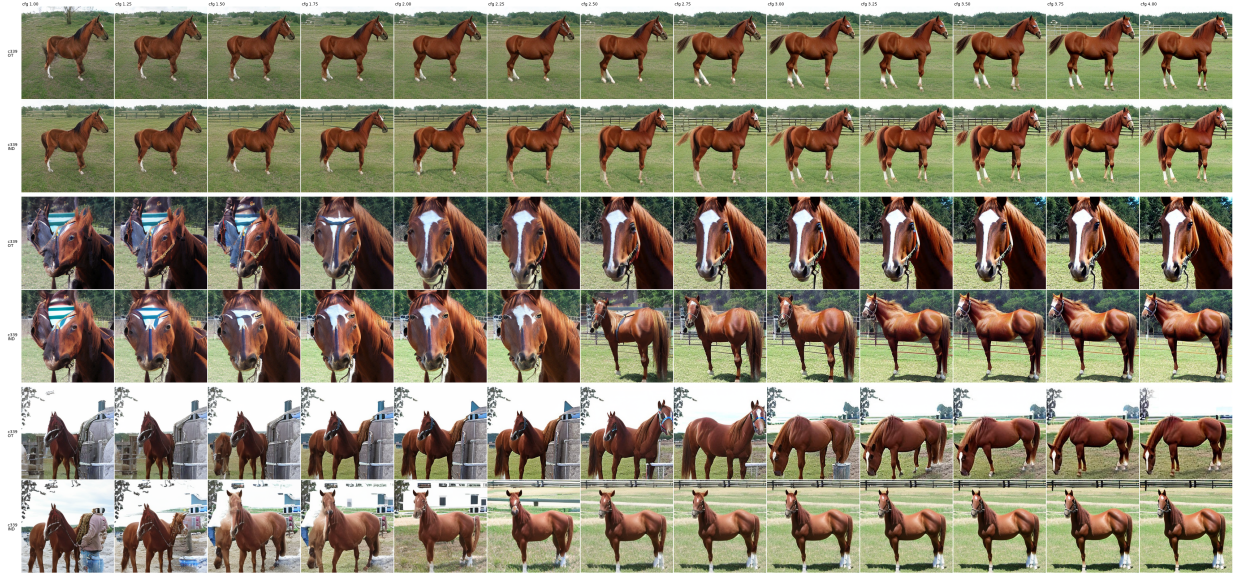}
    \caption{ SiT-XL/2 curated samples for class 339 \texttt{sorrel} across guidance scales $w$ and coupling choices. Under \gt the sample object characteristics remain stable, while for independent and class-conditional object changes its composition, pose, background at higher $w$.}
    \label{fig:mode_stability}
\end{figure}

\begin{figure}[t]
    \centering
    \includegraphics[width=\linewidth]{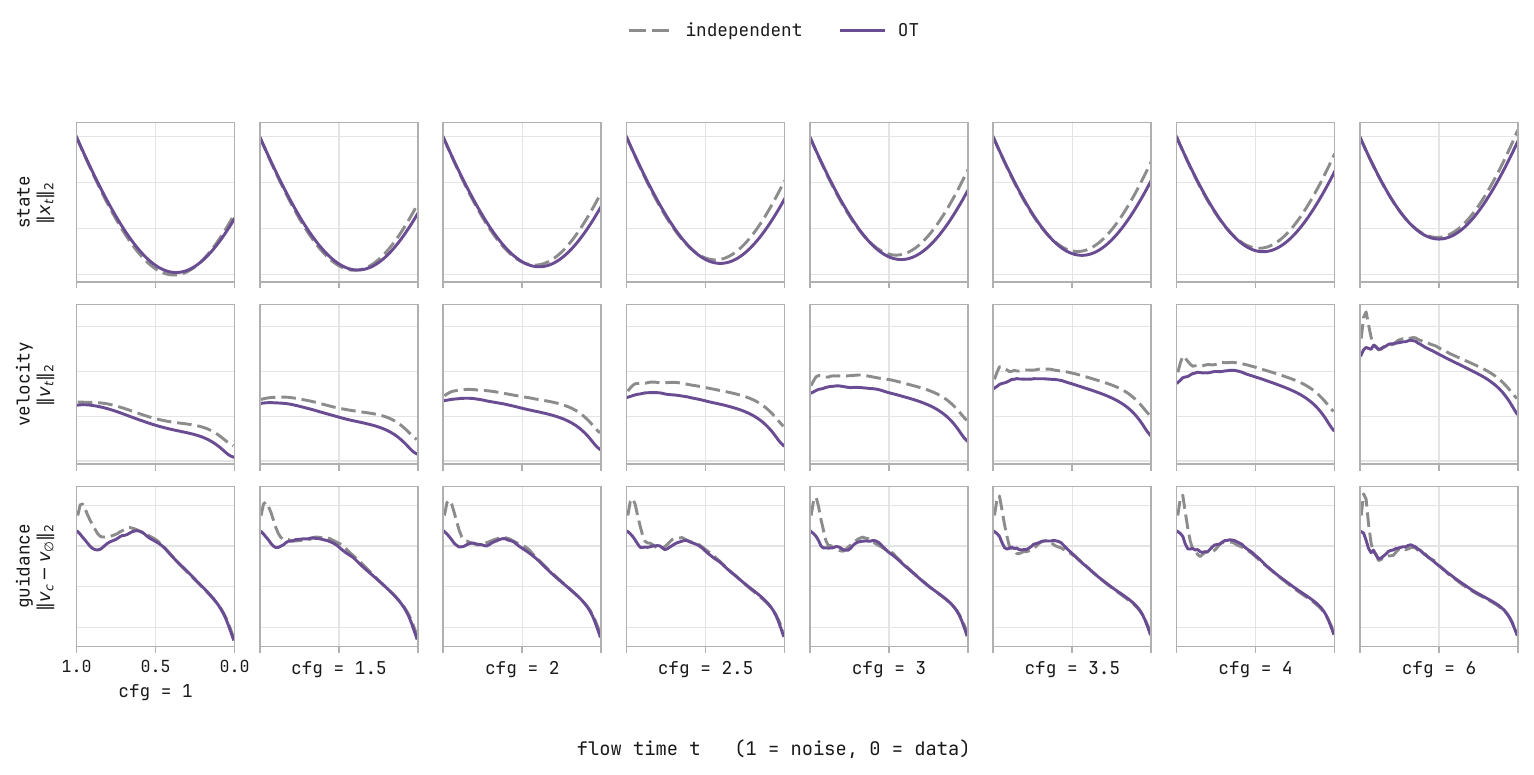}
    \caption{Measuring yo-yo effect and prediction gap across trajectory using SiT-XL/2 (64 steps).}
    \label{fig:placeholder}
\end{figure}

\xhdr{Guidance interval tuning} We show that \gt achieves strongest results overall with applied guidance interval tuning across B/2, L/2 and XL/2 scales in table~\ref{tab:gi_sweep}.

\begin{table*}[h]
\caption{\small FID\,$\downarrow$ across model scales and couplings, under the full guidance interval and tuned interval $[0, 0.7]$.\looseness=-1}
\label{tab:gi_sweep}
\centering
\scriptsize
\setlength{\tabcolsep}{3.5pt}
\begin{tabular}{ll c ccccc ccccc}
\toprule
 & & & \multicolumn{5}{c}{Full} & \multicolumn{5}{c}{GI [0, 0.7]} \\
\cmidrule(lr){4-8}\cmidrule(lr){9-13}
Model & Coupling & $w{=}1.0$ & 1.5 & 2.0 & 2.5 & 3.0 & 4.0 & 1.5 & 2.0 & 2.5 & 3.0 & 4.0 \\
\midrule
\multirow{3}{*}{SiT-B/2}
 & Independent     & \textbf{26.44} & \textbf{6.68} & 5.16 & 7.93 & 11.04 & 15.76 & \textbf{11.93} & \textbf{5.95} & 3.95 & 3.72 & 5.07 \\
 & Class-cond.\ OT & 26.72 & 6.70 & 5.08 & 7.81 & 10.90 & 15.59 & 12.13 & 6.00 & \textbf{3.92} & 3.68 & 5.00 \\
 & \gt             & 31.16 & 8.13 & \textbf{4.15} & \textbf{5.88} & \textbf{8.72} & \textbf{13.41} & 15.63 & 8.03 & 4.71 & \textbf{3.52} & \textbf{3.70} \\
\midrule
\multirow{2}{*}{SiT-L/2}
 & Independent & \textbf{13.85} & \textbf{2.95} & 5.85 & 10.10 & 13.53 & 17.89 & \textbf{4.53} & \textbf{2.35} & 2.62 & 3.65 & 6.00 \\
 & \gt         & 17.72 & 2.96 & \textbf{3.88} & \textbf{7.38} & \textbf{10.66} & \textbf{15.17} & 6.64 & 2.96 & \textbf{2.17} & \textbf{2.48} & \textbf{4.02} \\
\midrule
\multirow{2}{*}{SiT-XL/2}
 & Independent & \textbf{12.58} & 2.78 & 6.10 & 10.41 & 13.76 & 17.87 & \textbf{3.96} & \textbf{2.16} & 2.63 & 3.73 & 6.04 \\
 & \gt         & 16.11 & \textbf{2.63} & \textbf{3.87} & \textbf{7.47} & \textbf{10.70} & \textbf{15.09} & 5.89 & 2.61 & \textbf{2.01} & \textbf{2.40} & \textbf{3.91} \\
\bottomrule
\end{tabular}
\end{table*}

\newpage
\section{Single-cell Datasets}
\label{appendix:single_cell}
\xhdr{PBMC3K} 2,638 peripheral blood mononuclear cells from a healthy donor across 8 cell types.

\xhdr{Dentate gyrus} 18,213 cells from the developing mouse hippocampus~\parencite{la2018rna}, annotated with 14 cell types.


\xhdr{HLCA} 584,944 human lung cells from 486 individuals across 49 datasets~\parencite{sikkema2023integrated}, annotated with 50 cell types.

\xhdr{Implementation Details} We add couplings on top of the uni-modal generation set-up described  in \textcite{palma2025multimodal}. Within the single-cell dataset, CFGen splits data into 90\% training and 10\% validation. Evaluation is performed on the corresponding held-out test data. Since CFGen~\parencite{palma2025multimodal} considers interpolant $x_t = tx_1 + \sigma_t\epsilon$, $\epsilon\sim \mathcal{N}(0, I_d)$, we implement baseline with linear interpolant used in our image experiments (CFGen-linear) which trains CFGen with independent coupling. We then compare both baselines to \gt transport. Our evaluation script uses $0, 1, 2$ seeds.

CFGen splits its training into two stages: 1) encoding single-cell data with an RNA autoencoder and 2) training a conditional flow matching model in latent space. We show hyperparameter choices for training and backbone model in tables \ref{tab:cfgen_training} and \ref{tab:cfgen_resnet_configs}. Note that PBMC3k and Dentate gyrus use \texttt{resnet small} and HLCA uses \texttt{resnet big}.

\begin{table}[h]
\centering
\small
\caption{Training configuration for uni-modal generation with CFGen.}
\label{tab:cfgen_training}
\begin{tabular}{lcc}
\toprule
Parameter & Autoencoder & Latent Flow Matching \\
\midrule
Batch size       & 256       & 256 (64 for PBMC3K) \\
Epochs           & 300       & 1500 \\
Optimizer        & AdamW     & AdamW \\
Learning rate    & $1\times10^{-3}$ & $1\times10^{-4}$ \\
Weight decay     & $1\times10^{-5}$ & $1\times10^{-6}$ \\
Gradient clipping & 1.0      & 1.0 \\
Train/validation split & 90\% / 10\% & 90\% / 10\% \\
\bottomrule
\end{tabular}
\end{table}

\begin{table}[t]
\centering
\caption{CFGen ResNet Backbone.}
\label{tab:cfgen_resnet_configs}
\begin{tabular}{lcc}
\toprule
Hyperparameter & ResNet Small & ResNet Big \\
\midrule
Hidden dimension     & 32  & 64  \\
Residual blocks      & 3   & 3   \\
Embedding dimension  & 20  & 100 \\
Dropout probability  & 0.0 & 0.0 \\
Condition dropout probability         & 0.2 & 0.2 \\
\bottomrule
\end{tabular}
\end{table}

\newpage
\section{Additional Background}

\subsection{Classifier and classifier-free guidance}
\label{appendix: cfg_theory}
\xhdr{Classifier Guidance}
By applying Bayes' Rule we can write
\begin{equation}
    p_t(x\mid c)=\frac{p_t(c\mid x)p_t(x)}{p_t(c)}
\end{equation}
Taking a logarithm of each side we get
\begin{equation}
    \log p_t(x\mid c)=\log p_t(c\mid x)+\mathrm{log}p_t(x)-\log p_t(c)
\end{equation}
By applying $\nabla_x$ to both sides we get
\begin{equation}
    \nabla_x\log p_t(x\mid c)=\nabla_x\log p_t(c\mid x)+\nabla_x\log p_t(x)-\cancel{\nabla_x\log p_t(c)}
\end{equation}
leading to 
\begin{equation}
    \nabla_x\log p_t(x\mid c)=\nabla_x\mathrm{log}p_t(c\mid x)+\nabla_x\log p_t(x)
\end{equation}
Using the conversion formula for score we can further write this as 
\begin{equation}
    v_c(x,t,c)=v_u(x_t,t)+a_t\nabla_x\log p_t(c\mid x)
\end{equation}
Classifier guidance constructs a velocity field by enhancing classifier $p_t(c\mid x)$ with scaling factor $w$ which we refer to as guidance scale, yielding
\begin{equation}
\label{eq:CF_final}
    \boxed{v_c^{CF}(x,t,c)=v_u(x_t,t)+wa_t\nabla_x\log p_t(c\mid x)}
\end{equation}
where $a_t = \frac{1-t}{t}$ and $b_t=\frac{1}{t}$.

The main issue with using classifier guidance is the need to train $p_t(c\mid x)$ classifier, which drastically increases training compute. In the following section, we derive classifier-free guidance. 

\xhdr{Classifier-free Guidance} Classifier-free guidance takes a step further to construct a guided field that does not require additional training of classifier $p_t(c \mid x)$.

We can add $b_tx_t$ and subtract from the RHS of \eqref{eq:CF_final}
\begin{align}
    v^{\mathrm{CFG}}(x_t,t,c)
    &= v_u(x_t,t) + w\, a_t \nabla_x \log p_t(c \mid x_t)
    \nonumber\\
    &= v_u(x_t,t) + w\, a_t \nabla_x \log \frac{p_t(x_t \mid c)\, p_t(c)}{p_t(x_t)}
    \nonumber\\
    &= v_u(x_t,t) + w \big[ a_t \nabla_x \log p_t(x_t \mid c)
        + \cancel{a_t \nabla_x \log p_t(c)}
        - a_t \nabla_x \log p_t(x_t) \big]
    \nonumber\\
    &= v_u(x_t,t) + w \big[
        \underbrace{a_t \nabla_x \log p_t(x_t \mid c) + b_t x_t}_{v_c(x_t,t,c)}
        - \underbrace{\big(a_t \nabla_x \log p_t(x_t) + b_t x_t\big)}_{v_u(x_t,t)}
    \big]
\end{align}
which leads to 
\begin{equation}
\label{eq:CFG_final}
    \boxed{v_c^{CFG}(x,t,c)=v_u(x_t,t)+w\underbrace{(v_c(x,t,c)-v_u(x,t))}_{g(x,t,c)}}
\end{equation}
\newpage
\section{\gt Uncurated Samples}
\begin{figure}[h]
    \centering
    \includegraphics[width=0.95\linewidth]{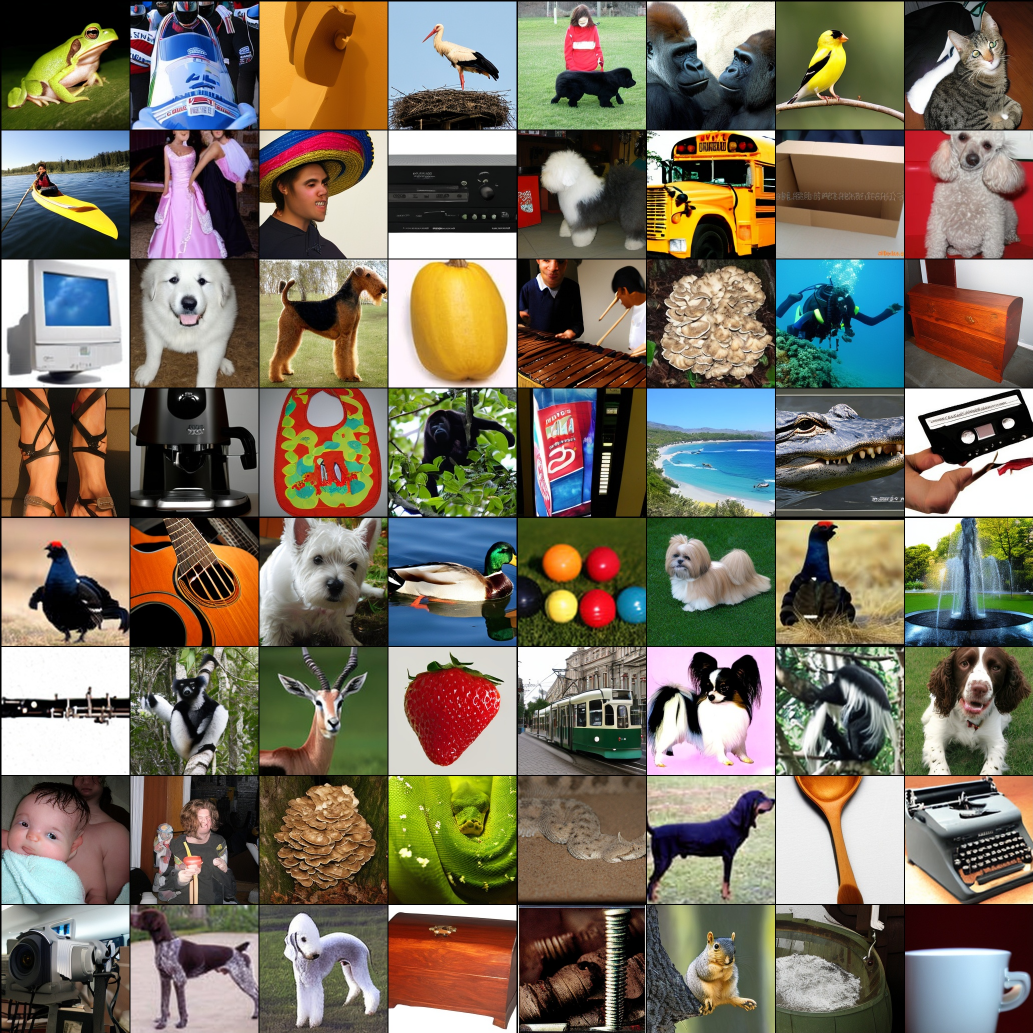}
    \caption{Samples generated with SiT-L/2, Euler (16 steps), $w=4.0$}
    \label{fig:placeholder}
\end{figure}
\newpage
\section{Class-conditional OT Uncurated Samples}
\begin{figure}[h]
    \centering
    \includegraphics[width=0.95\linewidth]{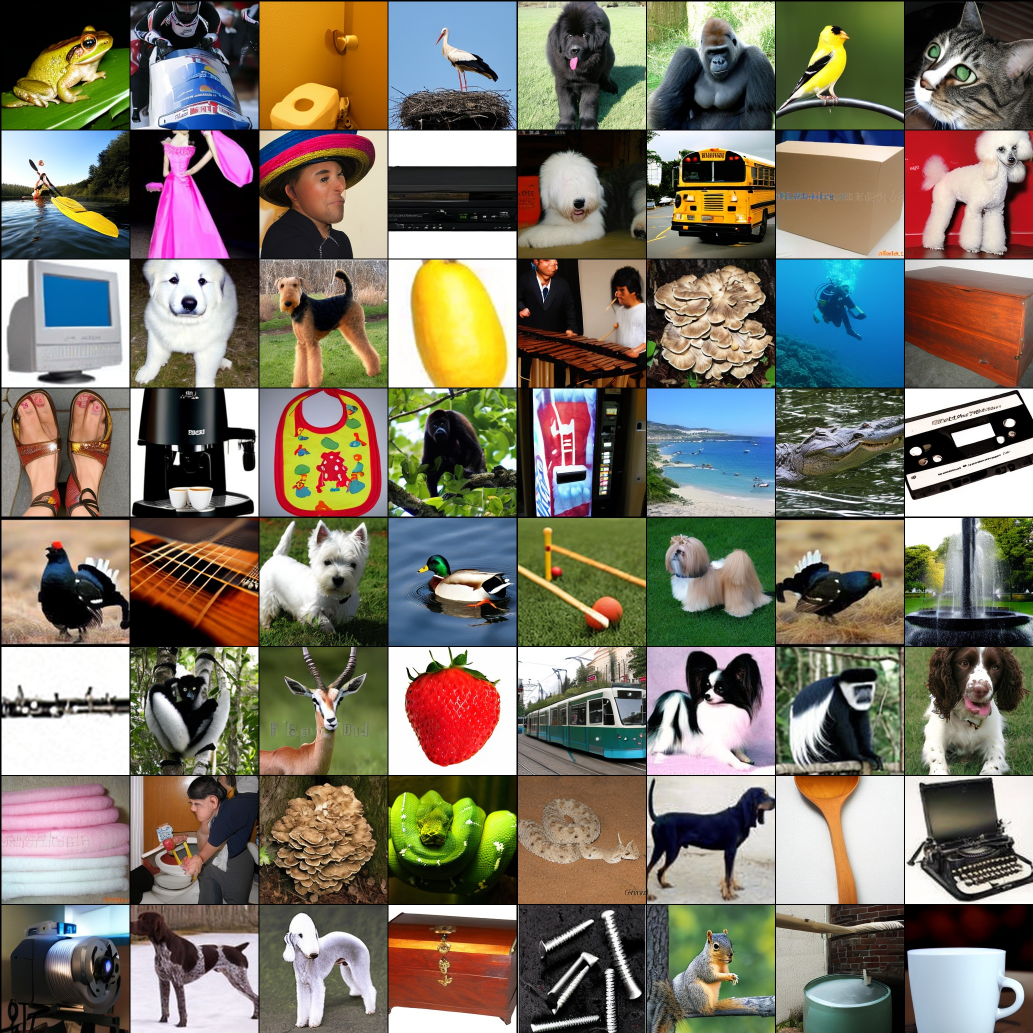}
    \caption{Samples generated with SiT-L/2, Euler (16 steps), $w=4.0$}
    \label{fig:placeholder}
\end{figure}

\newpage
\section{Independent Coupling Uncurated Samples}
\begin{figure}[h]
    \centering
    \includegraphics[width=0.95\linewidth]{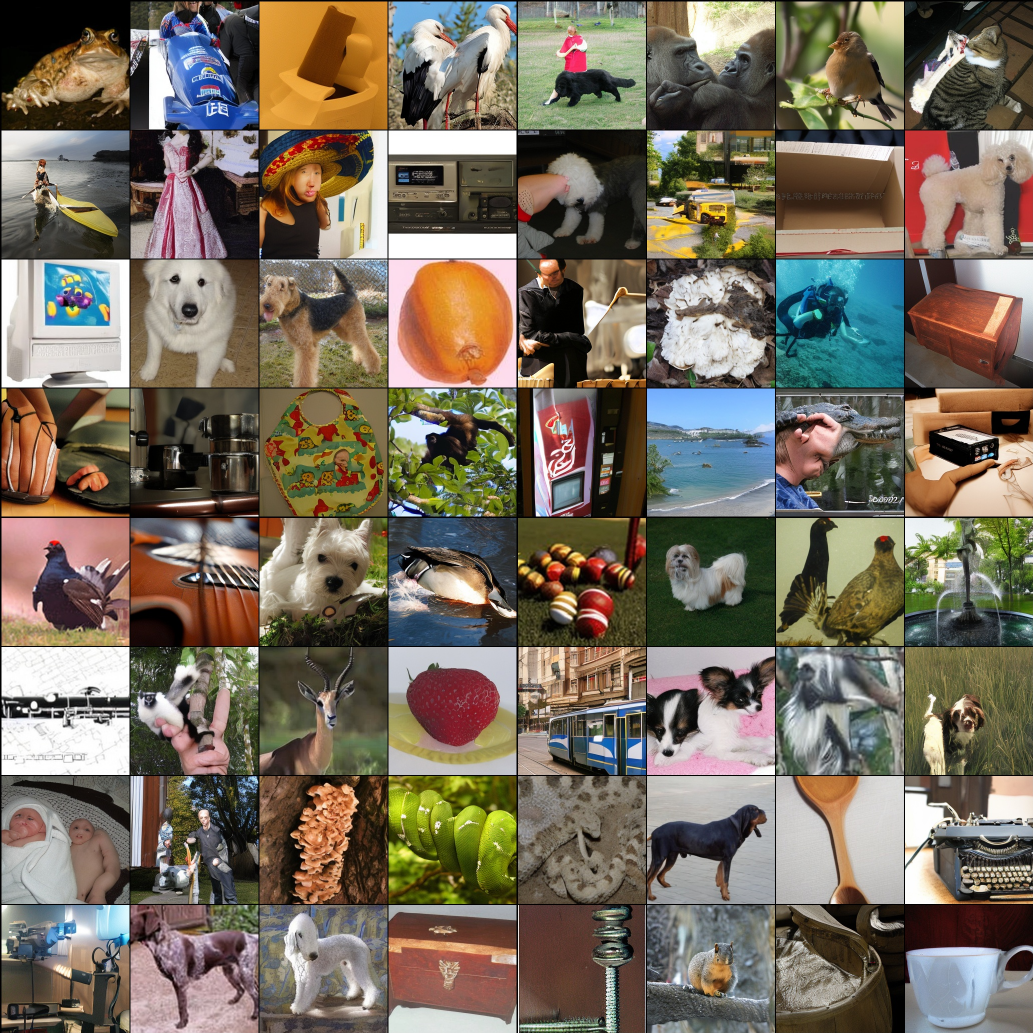}
    \caption{Samples generated with SiT-L/2, Euler (16 steps), $w=4.0$}
    \label{fig:placeholder}
\end{figure}
\newpage
\section{Text-conditioned ImageNet-256 Uncurated Samples} 
We show samples for the following prompts read left to right:\{\texttt{two mittens with colorful yarn on a bed}, \texttt{a white keyboard with a small white keypad}, \texttt{a sink with a white marble top and a black base}, \texttt{a green and purple flower with a black center}, \texttt{three monkeys sitting on a wooden bench}, \texttt{a black and white photo of a stethoscope}, \texttt{a dog is looking at ducks in a pond}, \texttt{a woman feeding her baby with a bottle}, \texttt{a bowl of food}, \texttt{a dog sitting on the grass with a red building in the background}, \texttt{a rocky cliff}, \texttt{a snake is laying on top of hay in a cage}, \texttt{a plate of food on a table}, \texttt{a white monkey hanging on a wooden pole}, \texttt{casio dg-2000 boombox}, \texttt{a dog laying in the grass}, \texttt{a bathroom with a toilet and a bathtub with gummy bears}, \texttt{a chameleon is sitting on a branch with green leaves}, \texttt{a group of masks with different colors and designs}, \texttt{a monkey sitting in a tree with leaves}, \texttt{a person walking along a road near a body of water}, \texttt{a wolf laying on the ground}, \texttt{a black and white photo of a wheel}, \texttt{a large dam with a large waterfall in front of it}, \texttt{a lab coat with a picture of a man in a lab coat}, \texttt{a female swimmer in the pool with a yellow cap}, \texttt{a brown and white dog with a collar on}, \texttt{a plate with shrimp and bacon}, \texttt{a pair of scissors on a wooden table}, \texttt{a bird with a long beak sitting on a branch}, \texttt{a spider sits on its web in the sun}, \texttt{a tree with a branch}, \texttt{a group of people standing in a line}, \texttt{a pool table in an empty building with a green light}, \texttt{panasonic pd-wg-g1}, \texttt{a meerkat standing on a rock in a zoo}, \texttt{a small hamster sleeping in a person's hand}, \texttt{a woman looking at a dinosaur in a museum}, \texttt{a knitting dish cloth and a knitting needle}, \texttt{a stone wall in the middle of a yard}, \texttt{a white and yellow sea slug on a coral reef}, \texttt{a large cicada sitting on a person's hand}, \texttt{two women in kimono}, \texttt{a vintage sewing machine sitting in the grass}, \texttt{a couple standing in front of a yurt}, \texttt{a man lifting a barbell on a competition stage}, \texttt{a circular clock with a white circle in the middle}, \texttt{a military vehicle with a gun mounted on top}, \texttt{three graduates pose for a photo in blue graduation gowns}, \texttt{two cars driving on a race track}, \texttt{a woman holding a large fruit}, \texttt{a large black and white whale with its tail out in the water}, \texttt{a wall of banjos hanging on a wall}, \texttt{a bowl of guacamole with a tortilla chip on top}, \texttt{a close up of a dog with a collar}, \texttt{two dogs are standing in the dirt}, \texttt{a dock with a concrete wall and palm trees}, \texttt{a plate of mashed potatoes}, \texttt{a pair of sunglasses on the ground}, \texttt{a large building with many people walking around it}, \texttt{a small orange fish in a bowl}, \texttt{a close up of a metal fan with a metal cover}, \texttt{three monkeys sitting on a log}, \texttt{a skillet with food in it}\}

\newpage
\subsection{\gt text-conditioned ImageNet-256}
\begin{figure}[h]
    \centering
    \includegraphics[width=0.95\linewidth]{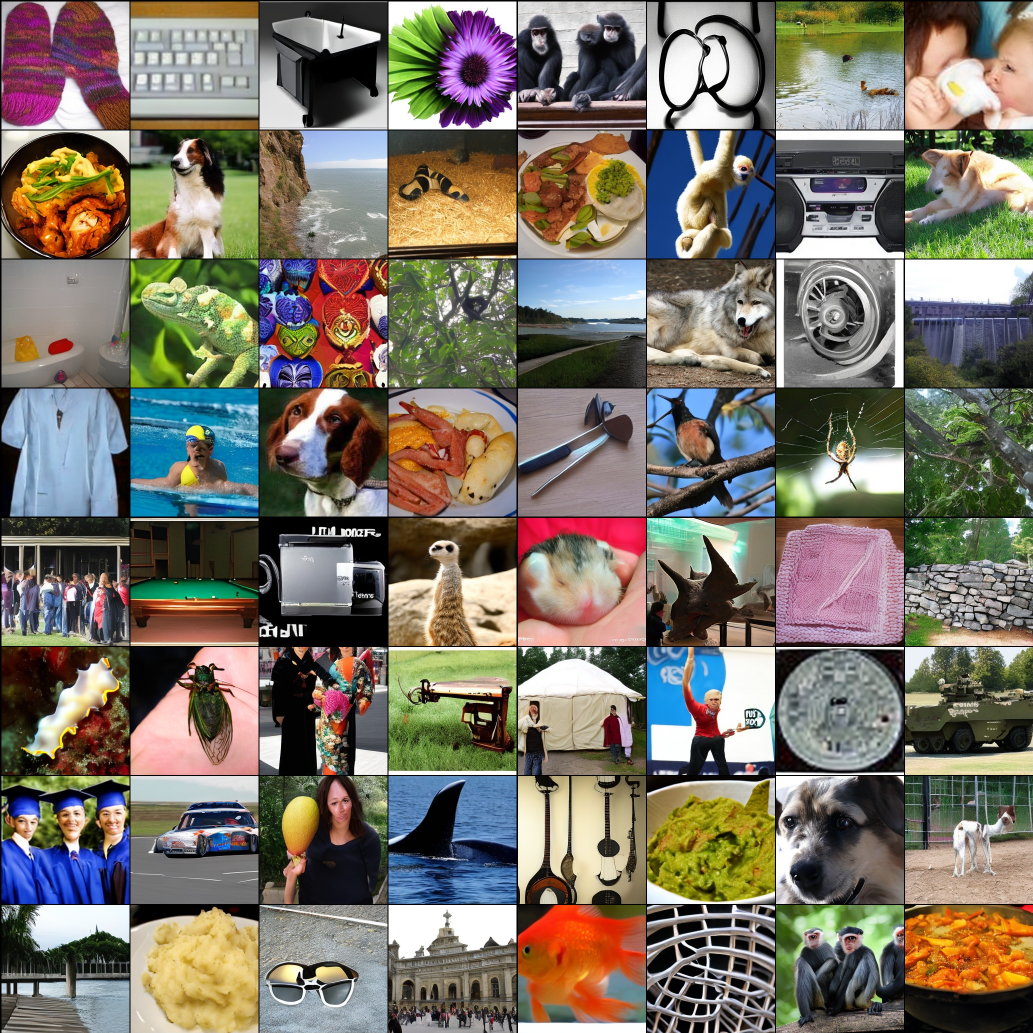}
    \caption{Samples generated with SiT-B/2, Euler (16 steps), $w=4.0$}
    \label{fig:placeholder}
\end{figure}
\newpage
\subsection{Independent text-conditioned ImageNet-256}
\begin{figure}[h]
    \centering
    \includegraphics[width=0.95\linewidth]{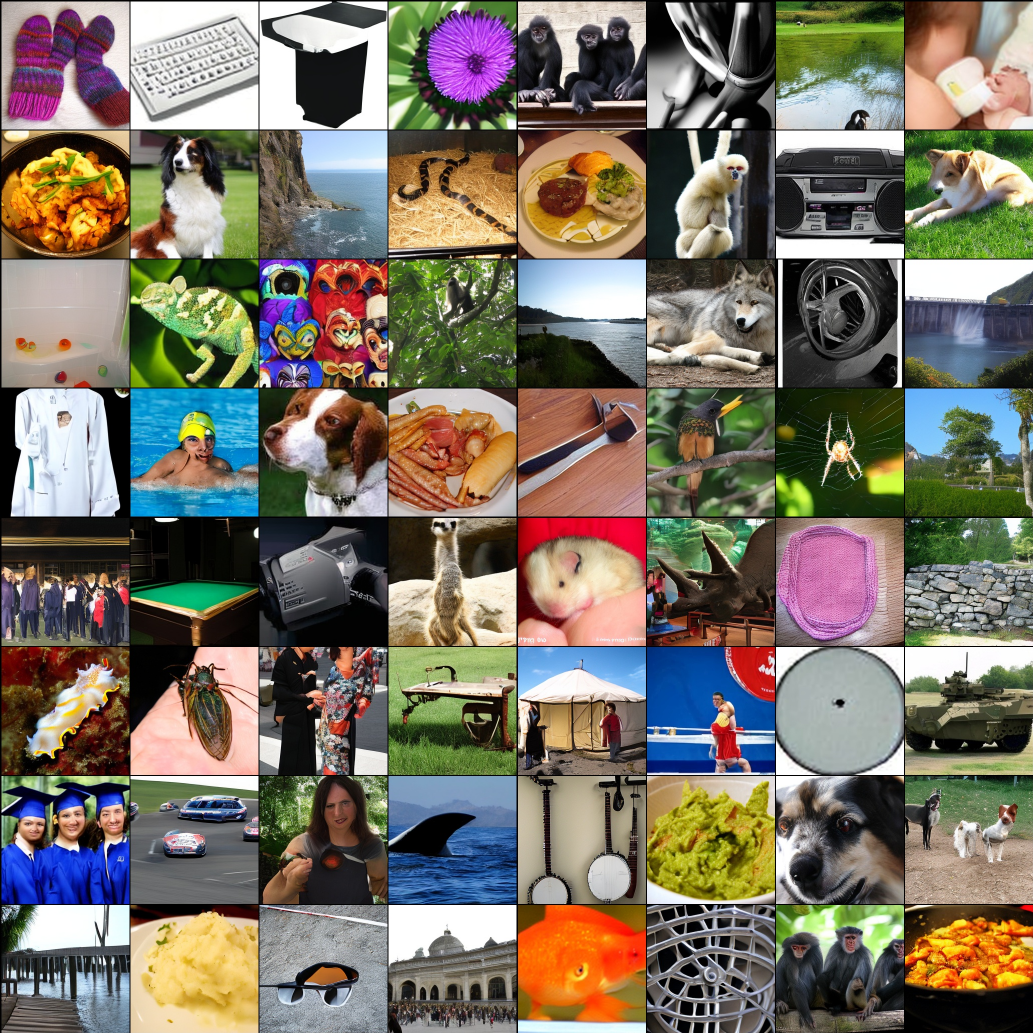}
    \caption{Samples generated with SiT-B/2, Euler (16 steps), $w=4.0$}
    \label{fig:placeholder}
\end{figure}
\end{document}

%% file: math_commands.tex
\usepackage{amsmath,amsfonts,bm}

\def\eqref#1{equation~\ref{#1}}

\def\1{\bm{1}}

\DeclareMathAlphabet{\mathsfit}{\encodingdefault}{\sfdefault}{m}{sl}
\SetMathAlphabet{\mathsfit}{bold}{\encodingdefault}{\sfdefault}{bx}{n}

\def\gL{{\mathcal{L}}}

%% file: macros.tex
\newcommand{\xhdr}[1]{{\noindent\bfseries #1.}}

\usepackage[framemethod=TikZ]{mdframed}
\mdfdefinestyle{MyFrame}{%
    linecolor=aithyra-black,
    outerlinewidth=0.5pt,
    roundcorner=0pt,
    innertopmargin=2pt, 
    innerbottommargin=2pt, 
    innerrightmargin=7pt,
    innerleftmargin=7pt,
    backgroundcolor=aithyra-white}

\mdfdefinestyle{MyFrame2}{%
    linecolor=aithyra-lightgrey,
    outerlinewidth=1pt,
    roundcorner=0pt,
    innertopmargin=10pt,
    innerbottommargin=7pt,
    innerrightmargin=10pt,
    innerleftmargin=10pt,
    backgroundcolor=aithyra-lightgrey}

\mdfdefinestyle{MyFrame3}{%
    linecolor=aithyra-lightgrey,
    outerlinewidth=1pt,
    roundcorner=0pt,
    innertopmargin=10pt,
    innerbottommargin=7pt,
    innerrightmargin=10pt,
    innerleftmargin=10pt,
    backgroundcolor=aithyra-lightgrey}

\mdfdefinestyle{MyFrameTeal}{%
    linecolor=aithyra-teal,
    linewidth=2pt,
    topline=false, bottomline=false, rightline=false, leftline=true,
    roundcorner=0pt,
    innertopmargin=10pt,
    innerbottommargin=7pt,
    innerrightmargin=10pt,
    innerleftmargin=10pt,
    backgroundcolor=aithyra-lightgrey!50}

\mdfdefinestyle{MyFrameEq}{%
    linecolor=aithyra-lightgrey,
    outerlinewidth=0pt,
    roundcorner=0pt,
    innertopmargin=0pt,
    innerbottommargin=0pt,
    innerrightmargin=7pt,
    innerleftmargin=7pt,
    backgroundcolor=aithyra-lightgrey}

%% file: gaussian_minibatch_appendix.tex
\label{appendix:gaussian_plugin}

\paragraph{Gaussian plug-in minibatch model} We study minibatch OT flow matching using an analytically tractable Gaussian
model. Let $p_0=\Nc(0,\Id)$ and $p_1=\Nc(0,\Sigma)$ be the source and target
distributions.  For independent minibatches of size $n$ from $p_0$ and $p_1$,
we fit the source and target means and covariances, couple the fitted Gaussians
by their Gaussian Monge map, and marginalize the resulting conditional
flow-matching fields over the randomly drawn minibatches.  We emphasize that
this is a \emph{Gaussian plug-in model} of mini-batch transport that we use
for its analytical tractability, since exact analysis of the discrete
Hungarian coupling used in our experiments would likely be significantly more
involved.

Writing $\widehat p_{t,n}(x)$ and $\widehat v_{t,n}(x)$ for the path density
and velocity induced by a batch pair, the marginal flow field learned by flow
matching is
\begin{equation}
    v_{t,n}(x)
    =
    \frac{\mathbb{E}
    [\widehat p_{t,n}(x)\widehat v_{t,n}(x)]}
    {\mathbb{E}[\widehat p_{t,n}(x)]}.
    \label{eq:gaussian-plugin-marginal-field}
\end{equation}
Here and throughout this section, these expectations are understood to be taken over the independently sampled source and target minibatches.
This posterior density weighting creates a nonlinear correction even though every batch-conditional field is affine.

Since $p_0$ is isotropic, we may assume that the target has diagonal covariance \(\Sigma=\operatorname{diag}(\sigma_1^2,\ldots,\sigma_d^2)\), \(\sigma_i>0\).
Define  \(m_i(t)=(1-t)+t\sigma_i\),
and the coefficients
\begin{align}
    \theta_{ij}(t)
    &:=
    \frac{\sigma_i\sigma_j}{\sigma_i+\sigma_j}
    \left[
      \frac{2t\sigma_i\sigma_j}{\sigma_i+\sigma_j}
      \bigl(m_i(t)+m_j(t)\bigr)-m_i(t)m_j(t)
    \right],
    \label{eq:gaussian-plugin-theta}
    \\
    \kappa_i(t)
    &:=-\frac{\sigma_i^2}{m_i(t)^2}
      \sum_{j\ne i}
      \frac{\sigma_j(2t\sigma_j-m_j(t))}
      {(\sigma_i+\sigma_j)^2m_j(t)}.
    \label{eq:gaussian-plugin-kappa}
\end{align}

\begin{proposition}[Gaussian plug-in minibatch field]
\label{prop:gaussian-plugin-general}
For fixed $d$, positive-definite diagonal $\Sigma$ as above, and fixed $(x,t)$,
the marginalized field (\eqref{eq:gaussian-plugin-marginal-field}) satisfies, as $n\to\infty$,
\begin{equation}
    [v_{t,n}(x)]_i
    =\frac{\sigma_i-1}{m_i(t)}x_i
     +\frac{1}{n}\left[
       \kappa_i(t)x_i
       +\frac{x_i}{m_i(t)^3}
        \sum_{j=1}^d\theta_{ij}(t)\frac{x_j^2}{m_j(t)^3}
     \right]
     +o(n^{-1}).
    \label{eq:gaussian-plugin-general-result}
\end{equation}
\end{proposition}

\begin{corollary}[Standardized endpoints]
\label{cor:gaussian-plugin-standardized}
If $\Sigma=\Id$, then
\begin{align}
    v_{t,n}(x)
    &=\frac{2t-1}{4n}
      \left(2\lVert x\rVert^2-(d-1)\right)x+o(n^{-1}),
    \label{eq:gaussian-plugin-standardized-velocity-main}
    \\
    \partial_t v_{t,n}(x)
    &=\frac{1}{2n}
      \left(2\lVert x\rVert^2-(d-1)\right)x+o(n^{-1}).
    \label{eq:gaussian-plugin-standardized-acceleration-main}
\end{align}
\end{corollary}

\begin{proposition}[Independent Gaussian flow-matching field]
\label{prop:gaussian-independent}
For independent flow matching with
$p_0=\Nc(0,\Id)$ and
$p_1=\Nc(0,\operatorname{diag}(\sigma_1^2,\ldots,\sigma_d^2))$, define \(q_i(t)=(1-t)^2+t^2\sigma_i^2\), \(r_i(t)=t\sigma_i^2-(1-t)\).
Then
\begin{align}
    [v_t^{\mathrm{ind}}(x)]_i
    &=\frac{r_i(t)}{q_i(t)}x_i,
    \label{eq:gaussian-plugin-independent-velocity-main}
    \\
    [\partial_t v_t^{\mathrm{ind}}(x)]_i
    &=\left[
      \frac{1+\sigma_i^2}{q_i(t)}
      -\frac{2r_i(t)^2}{q_i(t)^2}
    \right]x_i.
    \label{eq:gaussian-plugin-independent-time-derivative-main}
\end{align}
If $\Sigma=\Id$, these expressions specialize to
\begin{equation}
    v_t^{\mathrm{ind}}(x)
    =\frac{2t-1}{(1-t)^2+t^2}x,
    \qquad
    \partial_t v_t^{\mathrm{ind}}(x)
    =\frac{4t(1-t)}{((1-t)^2+t^2)^2}x.
    \label{eq:standardized-independent-main}
\end{equation}
\end{proposition}

The population Gaussian OT field is
$[v_t^\star(x)]_i=(\sigma_i-1)x_i/m_i(t)$.  Unlike the plug-in correction,
the independent field in Proposition~\ref{prop:gaussian-independent} has no
batch-size dependence.  When $\Sigma=\Id$, population OT is the identity and
has zero velocity and time derivative, whereas independent FM exhibits an
order-one contraction--expansion (referred to in the main paper as the ``yo-yo'' effect) despite having identical endpoint distributions.
Gaussian plug-in minibatching retains a residual version of this motion, but
Corollary~\ref{cor:gaussian-plugin-standardized} shows that its magnitude and
Eulerian time derivative decay as $1/n$.

We next prove Proposition~\ref{prop:gaussian-plugin-general} and its corollary;
Proposition~\ref{prop:gaussian-independent} follows from the joint Gaussian
conditioning calculation given afterward.

\xhdr{Setup}
We keep $d$ and $\Sigma$ fixed as $n\to\infty$ and assume $n>d$.  With
$\varepsilon=n^{-1/2}$ and unbiased sample covariances, write the source and target parameters as
\begin{align}
    \widehat\mu_n^0 &=\varepsilon h,
    &
    \widehat\Sigma_n^0 &=\Id+\varepsilon H,
    \\
    \widehat\mu_n^1 &=\varepsilon\Sigma^{1/2}h',
    &
    \widehat\Sigma_n^1
    &=\Sigma+\varepsilon\Sigma^{1/2}H'\Sigma^{1/2}.
    \label{eq:gaussian-plugin-sample-fluctuations}
\end{align}
Here $h,h'$ are independent standard Gaussian vectors, while $H,H'$ are
independent centered Wishart fluctuations and are independent of the sample
means.  Their leading second moments are
\begin{equation}
    \mathbb{E}[H_{ij}H_{kl}]
    =\delta_{ik}\delta_{jl}+\delta_{il}\delta_{jk}+O(n^{-1}),
    \label{eq:gaussian-plugin-wishart-moments}
\end{equation}
and likewise for $H'$.  The $O(n^{-1})$ term in
\eqref{eq:gaussian-plugin-wishart-moments} contributes only beyond the order
retained below.

The Monge map between the two fitted Gaussians is
\begin{equation}
    T_n(x)=\widehat\mu_n^1+A_n(x-\widehat\mu_n^0),
    \qquad
    A_n\widehat\Sigma_n^0A_n=\widehat\Sigma_n^1,
    \qquad A_n\succ0.
\end{equation}
Expanding
$A_n=\Sigma^{1/2}+\varepsilon K+\varepsilon^2L+O_p(\varepsilon^3)$
and matching powers in the covariance identity gives
\begin{align}
    \Sigma^{1/2}K+K\Sigma^{1/2}
    &=\Sigma^{1/2}(H'-H)\Sigma^{1/2},
    \\
    \Sigma^{1/2}L+L\Sigma^{1/2}
    &=-K^2-\Sigma^{1/2}HK-KH\Sigma^{1/2}.
    \label{eq:gaussian-plugin-sylvester-equations}
\end{align}
Since $\Sigma$ is diagonal, the first equation has the entrywise solution
\begin{equation}
    K_{ij}=c_{ij}(H'_{ij}-H_{ij}),
    \qquad
    c_{ij}:=\frac{\sigma_i\sigma_j}{\sigma_i+\sigma_j}.
    \label{eq:gaussian-plugin-K-entrywise}
\end{equation}

For later use, define
\begin{equation}
    M_t=(1-t)\Id+t\Sigma^{1/2},
    \qquad
    W_t=M_t^{-1}.
    \label{eq:gaussian-plugin-path-notation}
\end{equation}
The diagonal entries of $M_t$ are the scalars $m_i(t)$. 
The population Gaussian OT field is
\begin{equation}
    v_t^\star(x)=(\Sigma^{1/2}-\Id)W_tx,
    \qquad
    [v_t^\star(x)]_i=\frac{\sigma_i-1}{m_i(t)}x_i.
    \label{eq:gaussian-plugin-population-field}
\end{equation}

\xhdr{Batch-conditional and marginalized fields}
For a fixed pair of batch fits, set
\begin{equation}
    X_t=(1-t)X_0+tT_n(X_0),
    \qquad
    X_0\sim\Nc(\widehat\mu_n^0,\widehat\Sigma_n^0).
\end{equation}
Writing
\begin{equation}
    B_{t,n}=(1-t)\Id+tA_n,
    \qquad
    \widehat m_{t,n}
    =(1-t)\widehat\mu_n^0+t\widehat\mu_n^1,
\end{equation}
the conditional path is Gaussian and its FM field is the affine map
\begin{equation}
    \widehat v_{t,n}(x)
    =\widehat\mu_n^1-\widehat\mu_n^0
     +(A_n-\Id)B_{t,n}^{-1}(x-\widehat m_{t,n}).
    \label{eq:gaussian-plugin-conditional-field}
\end{equation}
If
\begin{equation}
    \alpha_t=(1-t)h+t\Sigma^{1/2}h',
    \qquad
    e_t=\Sigma^{1/2}W_t(h'-h),
\end{equation}
then expansion of \eqref{eq:gaussian-plugin-conditional-field} yields
\begin{equation}
    \widehat v_{t,n}(x)
    =v_t^\star(x)+\varepsilon v_t^{(1)}(x)
     +\varepsilon^2v_t^{(2)}(x)+O_p(\varepsilon^3),
    \label{eq:gaussian-plugin-conditional-expansion}
\end{equation}
where
\begin{align}
    v_t^{(1)}(x)
    &=W_tKW_tx+e_t,
    \\
    v_t^{(2)}(x)
    &=\left(W_tLW_t-tW_tKW_tKW_t\right)x-W_tKW_t\alpha_t.
    \label{eq:gaussian-plugin-v1-v2}
\end{align}

The marginalization in \eqref{eq:gaussian-plugin-marginal-field} weights each
batch-conditional field by $\widehat p_{t,n}(x)$.  We therefore expand this
conditional path density about $p_t^\star = \Nc(0,M_t^2)$.

Conditional on the fitted batches, the covariance of $X_t$
has the expansion
\begin{equation}
    B_{t,n}\widehat\Sigma_n^0B_{t,n}
    =M_t^2+\varepsilon G_t+O_p(\varepsilon^2),
    \qquad
    G_t=M_tHM_t+t(KM_t+M_tK).
\end{equation}
Consequently,
\begin{equation}
    \widehat p_{t,n}(x)
    =p_t^\star(x)\left[1+\varepsilon\ell_t(x)+O_p(\varepsilon^2)\right],
\end{equation}
with
\begin{equation}
    \ell_t(x)
    =\alpha_t^{\mathsf T}W_t^2x
     +\frac12x^{\mathsf T}W_t^2G_tW_t^2x
     -\frac12\operatorname{tr}(W_t^2G_t).
    \label{eq:gaussian-plugin-density-score}
\end{equation}
Since $\alpha_t$ and $G_t$ are centered, $\mathbb{E}[\ell_t(x)]=0$;
similarly, the centered fluctuations $K$ and $e_t$ give
$\mathbb{E}[v_t^{(1)}(x)]=0$.  Expanding the density-weighted ratio
\eqref{eq:gaussian-plugin-marginal-field} gives
\begin{equation}
    v_{t,n}(x)
    =v_t^\star(x)
     +\frac1n\left\{
       \mathbb{E}[v_t^{(2)}(x)]
       +\mathbb{E}[\ell_t(x)v_t^{(1)}(x)]
     \right\}
     +o(n^{-1}).
    \label{eq:gaussian-plugin-master-expansion}
\end{equation}
The second-order density fluctuation cancels between the numerator and
denominator.  The second expectation in
\eqref{eq:gaussian-plugin-master-expansion} is precisely the posterior
density-weighting term omitted by an unweighted batch average.

\xhdr{Closed-form coefficient}
We now evaluate the coefficient of the order-$1/n$ term in
\eqref{eq:gaussian-plugin-master-expansion}.  In the posterior-weighting term
$\mathbb{E}[\ell_t(x)v_t^{(1)}(x)]$, the covariance-dependent part of
$\ell_t$ is linear in $G_t$, while the matrix-dependent part of $v_t^{(1)}$
is linear in $K$.  Thus, using \eqref{eq:gaussian-plugin-K-entrywise} and the
Wishart moments in \eqref{eq:gaussian-plugin-wishart-moments}, their required
covariance contraction is
\begin{equation}
    \mathbb{E}[[G_t]_{ij}K_{kl}]
    =\theta_{ij}(t)
      (\delta_{ik}\delta_{jl}+\delta_{il}\delta_{jk})+O(n^{-1}).
    \label{eq:gaussian-plugin-GK-contraction}
\end{equation}
Combining this contraction with the second Sylvester equation in
\eqref{eq:gaussian-plugin-sylvester-equations}, which determines the mean of
$v_t^{(2)}$, gives
\begin{align}
    \mathbb{E}[v_t^{(2)}(x)]_i
    &=\beta_i(t)x_i+O(n^{-1}),
    \\
    \mathbb{E}[\ell_t(x)v_t^{(1)}(x)]_i
    &=\eta_i(t)x_i
      +\frac{x_i}{m_i(t)^3}
       \sum_{j=1}^d\theta_{ij}(t)\frac{x_j^2}{m_j(t)^3}
      -\frac{\theta_{ii}(t)}{m_i(t)^4}x_i+O(n^{-1}),
\end{align}
where
\begin{align}
    \beta_i(t)
    &=-\frac{\sigma_i^2}{m_i(t)^2}
      \sum_{j=1}^d(1+\delta_{ij})
      \frac{\sigma_j(2t\sigma_j-m_j(t))}
      {(\sigma_i+\sigma_j)^2m_j(t)},
    \\
    \eta_i(t)
    &=\frac{\sigma_i(2t\sigma_i-m_i(t))}{m_i(t)^3}.
\end{align}
The diagonal identities
\begin{equation}
    \beta_i(t)=\kappa_i(t)-\frac12\eta_i(t),
    \qquad
    \frac{\theta_{ii}(t)}{m_i(t)^4}=\frac12\eta_i(t)
\end{equation}
cancel all diagonal linear contributions.  Inserting the remaining terms into
\eqref{eq:gaussian-plugin-master-expansion} proves
\eqref{eq:gaussian-plugin-general-result}.

\xhdr{Standardized specialization}
When $\Sigma=\Id$, we have $\sigma_i=m_i(t)=1$ and
$2t\sigma_i-m_i(t)=2t-1$.  Equations~\eqref{eq:gaussian-plugin-theta}
and \eqref{eq:gaussian-plugin-kappa} reduce to
\begin{equation}
    \theta_{ij}(t)=\frac{2t-1}{2},
    \qquad
    \kappa_i(t)=-\frac{(d-1)(2t-1)}{4}.
\end{equation}
Substituting these expressions into
\eqref{eq:gaussian-plugin-general-result} gives
\eqref{eq:gaussian-plugin-standardized-velocity-main}; differentiating at
fixed $x$ gives \eqref{eq:gaussian-plugin-standardized-acceleration-main},
proving Corollary~\ref{cor:gaussian-plugin-standardized}.  At $t=1/2$ the
field is in fact exactly zero for every $n$: exchanging the i.i.d. fitted
batches preserves the midpoint and reverses the displacement.

In one dimension, the off-diagonal sum in
\eqref{eq:gaussian-plugin-kappa} is empty.  Writing
$m(t)=(1-t)+t\sigma$, the general result becomes
\begin{equation}
    v_{t,n}(x)
    =\frac{\sigma-1}{m(t)}x
     +\frac1n
      \frac{\sigma(2t\sigma-m(t))}{2m(t)^5}x^3
     +o(n^{-1}).
    \label{eq:gaussian-plugin-one-dimensional-result}
\end{equation}
Hence the entire linear order-$1/n$ correction cancels in one dimension.

\xhdr{Comparison with the independent coupling}
\begin{proof}[Proof of Proposition~\ref{prop:gaussian-independent}]
For independent
$X_0\sim\Nc(0,\Id)$ and $X_1\sim\Nc(0,\Sigma)$, let
$X_t=(1-t)X_0+tX_1$ and $U=X_1-X_0$.  Joint Gaussian conditioning yields
\begin{equation}
    v_t^{\mathrm{ind}}(x)
    =\mathbb{E}[U\mid X_t=x]
    =\left(t\Sigma-(1-t)\Id\right)
      \left((1-t)^2\Id+t^2\Sigma\right)^{-1}x.
    \label{eq:gaussian-plugin-independent-field}
\end{equation}
Diagonalizing this expression and differentiating at fixed $x$ gives
\eqref{eq:gaussian-plugin-independent-velocity-main} and
\eqref{eq:gaussian-plugin-independent-time-derivative-main}; setting
$\Sigma=\Id$ recovers \eqref{eq:standardized-independent-main}.
\end{proof}
The independent field is also exactly the field induced by singleton discrete
minibatch OT, because the sole source and target observations must be paired.
This does not identify it with Gaussian plug-in OT at $n=1$, for which an
unbiased sample covariance is undefined.

\xhdr{Acceleration along flow trajectories}
For a time-dependent velocity field $v_t$, define its material acceleration by
\begin{equation}
    a_t[v](x)
    :=\partial_t v_t(x)+D_xv_t(x)v_t(x).
    \label{eq:gaussian-material-acceleration-definition}
\end{equation}
This is the acceleration of a trajectory satisfying
$\dot x_t=v_t(x_t)$.  In particular, the population Gaussian OT field follows
straight displacement trajectories and therefore satisfies
$a_t[v^\star](x)=0$.

For the plug-in field, let
\begin{equation}
    b_i(t):=\frac{\sigma_i-1}{m_i(t)}
\end{equation}
and write overdots for time derivatives.  Expanding
\eqref{eq:gaussian-material-acceleration-definition} using
\eqref{eq:gaussian-plugin-general-result} gives
\begin{equation}
    \begin{aligned}
    [a_t[v_n](x)]_i
    =\frac{x_i}{n}\bigg[
       &\dot\kappa_i(t)+2b_i(t)\kappa_i(t)\\
       &+\sum_{j=1}^d
       \frac{\dot\theta_{ij}(t)
             -(b_i(t)+b_j(t))\theta_{ij}(t)}
            {m_i(t)^3m_j(t)^3}x_j^2
    \bigg]+o(n^{-1}).
    \end{aligned}
    \label{eq:gaussian-plugin-material-acceleration}
\end{equation}
For independent flow matching, the corresponding expression is exact:
\begin{equation}
    [a_t[v^{\mathrm{ind}}](x)]_i
    =\frac{\sigma_i^2}{q_i(t)^2}x_i.
    \label{eq:gaussian-independent-material-acceleration}
\end{equation}

When $\Sigma=\Id$, write $q(t)=(1-t)^2+t^2$.  The three acceleration fields
simplify to
\begin{equation}
    a_t[v^\star](x)=0,
    \qquad
    a_t[v_n](x)
    =\frac{2\lVert x\rVert^2-(d-1)}{2n}x+o(n^{-1}),
    \qquad
    a_t[v^{\mathrm{ind}}](x)=\frac{x}{q(t)^2}.
    \label{eq:gaussian-standardized-material-accelerations}
\end{equation}
The plug-in material acceleration agrees with the Eulerian derivative in
\eqref{eq:gaussian-plugin-standardized-acceleration-main} to order $1/n$,
because the convective term is of order $1/n^2$ in the standardized case.

Finally, let $x_t$ denote the flow trajectory initialized at $x_0$, and define
its integrated squared acceleration by
\begin{equation}
    \mathcal{A}[v;x_0]
    :=\int_0^1\lVert a_t[v](x_t)\rVert^2\,\mathrm{d}t.
\end{equation}
For standardized endpoints,
$x_t^\star=x_0$,
$x_t^{\mathrm{ind}}=\sqrt{q(t)}\,x_0$, and
$x_{t,n}^{\mathrm{plug}}=x_0+O(n^{-1})$.  Consequently,
\begin{align}
    \mathcal{A}[v^\star;x_0]
    &=0,
    \\
    \mathcal{A}[v_n;x_0]
    &=\frac{\lVert x_0\rVert^2
       \left(2\lVert x_0\rVert^2-(d-1)\right)^2}{4n^2}
      +o(n^{-2}),
    \label{eq:gaussian-plugin-integrated-acceleration}
    \\
    \mathcal{A}[v^{\mathrm{ind}};x_0]
    &=\left(2+\frac{3\pi}{4}\right)\lVert x_0\rVert^2.
    \label{eq:gaussian-independent-integrated-acceleration}
\end{align}
Thus the leading plug-in acceleration energy is nonincreasing in $n$ and
decays as $n^{-2}$, while independent FM incurs order-one acceleration and
exact OT incurs none.  For every fixed $x_0\ne0$, independent FM therefore has
the largest of these three acceleration energies for all sufficiently large
$n$.  This comparison concerns the controlled large-$n$ expansion and does
not assert exact monotonicity of the finite-$n$ Gaussian plug-in field.

If the initial condition is itself random, $X_0\sim\Nc(0,\Id)$, then
$R:=\lVert X_0\rVert^2\sim\chi_d^2$, with
\begin{equation}
    \mathbb{E}[R]=d,
    \qquad
    \mathbb{E}[R^2]=d(d+2),
    \qquad
    \mathbb{E}[R^3]=d(d+2)(d+4).
\end{equation}
Consequently, assuming the plug-in expansion above holds in $L^1$ with respect
to $X_0$,
\begin{align}
    \mathbb{E}\!\left[\mathcal{A}[v^\star;X_0]\right]
    &=0,
    \\
    \mathbb{E}\!\left[\mathcal{A}[v_n;X_0]\right]
    &=\frac{d\left(d^2+18d+41\right)}{4n^2}+o(n^{-2}),
    \label{eq:gaussian-plugin-expected-integrated-acceleration}
    \\
    \mathbb{E}\!\left[\mathcal{A}[v^{\mathrm{ind}};X_0]\right]
    &=\left(2+\frac{3\pi}{4}\right)d.
    \label{eq:gaussian-independent-expected-integrated-acceleration}
\end{align}
Thus, at fixed dimension, the expected plug-in acceleration energy again
decays as $n^{-2}$, whereas the independent-coupling energy remains order one
in $n$.